\documentclass[11pt]{article}

\usepackage{fullpage,times,url,natbib}
\usepackage{amsthm,amsmath,amssymb,epsfig,color,float,graphicx,verbatim,mathtools}
\usepackage{algorithm,algorithmic}
\usepackage{enumitem}

\usepackage{hyperref}
\hypersetup{
	colorlinks   = true, 
	urlcolor     = blue, 
	linkcolor    = blue, 
	citecolor   =  
}

\usepackage[utf8]{inputenc} 
\usepackage[T1]{fontenc}    
\usepackage{url}            
\usepackage{booktabs}       
\usepackage{amsfonts}       
\usepackage{nicefrac}       
\usepackage{microtype}      
\usepackage{xcolor}         
\usepackage{epstopdf}

\newtheorem{theorem}{Theorem}
\newtheorem{proposition}{Proposition}
\newtheorem{lemma}{Lemma}

\newtheorem{remark}{Remark}
\newtheorem{assumption}{Assumption}

\newcommand{\floor}[1]{\left\lfloor #1 \right\rfloor}
\newcommand{\p}[1]{\left( #1 \right)}
\newcommand{\pcc}[1]{\left[ #1 \right]}
\newcommand{\set}[1]{\left\lbrace#1\right\rbrace}
\newcommand{\abs}[1]{\left| #1 \right|}
\newcommand{\reals}{\mathbb{R}}
\newcommand{\E}{\mathbb{E}}

\newcommand{\bx}{\mathbf{x}}

\newcommand{\bb}{\mathbf{b}}

\newcommand{\bv}{\mathbf{v}}

\newcommand{\Ocal}{\mathcal{O}}

\newcommand{\norm}[1]{\|#1\|}

\newcommand{\lambdamax}{\lambda_{\max}}

\newcommand{\secref}[1]{Sec.~\ref{#1}}
\newcommand{\subsecref}[1]{Subsection~\ref{#1}}
\newcommand{\figref}[1]{Fig.~\ref{#1}}
\renewcommand{\eqref}[1]{Eq.~(\ref{#1})}
\newcommand{\lemref}[1]{Lemma~\ref{#1}}

\newcommand{\thmref}[1]{Thm.~\ref{#1}}
\newcommand{\propref}[1]{Proposition~\ref{#1}}
\newcommand{\appref}[1]{Appendix~\ref{#1}}

\title{
	Random Shuffling Beats SGD Only After Many Epochs\\on Ill-Conditioned Problems}

\author{
	Itay Safran \\
	Princeton University\\
	\texttt{isafran@princeton.edu} \\
	\and
	Ohad Shamir \\
	Weizmann Institute of Science \\
	\texttt{ohad.shamir@weizmann.ac.il} \\
}
\date{}
  
\begin{document}

\maketitle

\begin{abstract}
Recently, there has been much interest in studying the convergence rates of without-replacement SGD, and proving that it is faster than with-replacement SGD in the worst case. However, known lower bounds ignore the problem's geometry, including its condition number, whereas the upper bounds explicitly depend on it. Perhaps surprisingly, we prove that when the condition number is taken into account, without-replacement SGD \emph{does not} significantly improve on with-replacement SGD in terms of worst-case bounds, unless the number of epochs (passes over the data) is larger than the condition number. Since many problems in machine learning and other areas are both ill-conditioned and involve large datasets, this indicates that without-replacement does not necessarily improve over with-replacement sampling for realistic iteration budgets. We show this by providing new lower and upper bounds which are tight (up to log factors), for quadratic problems with commuting quadratic terms, precisely quantifying the dependence on the problem parameters. 
\end{abstract}


\section{Introduction}

We consider solving finite-sum optimization problems of the form
\[
F(\bx) ~=~ \frac{1}{n}\sum_{i=1}^{n}f_i(\bx)
\]
using stochastic gradient descent (SGD). Such problems are extremely common in modern machine learning (e.g., for empirical risk minimization), and stochastic gradient methods are the most popular approach for large-scale problems, where both $n$ and the dimension are large. The classical approach to apply SGD to such problems is to repeatedly sample indices $i\in \{1,\ldots,n\}$ uniformly at random, and perform updates of the form $\bx_{t+1}=\bx_t-\eta_t \nabla f_i(\bx_t)$, where $\eta_t$ is a step-size parameter. With this sampling scheme, each $\nabla f_i(\bx_t)$ is a random unbiased estimate of $\nabla F(\bx_t)$ (conditioned on $\bx_t$). Thus, the algorithm can be seen as a ``cheap'' noisy version of plain gradient descent on $F(\cdot)$, where  each iteration requires computing the gradient of just a single function $f_i(\cdot)$, rather than the gradient of $F(\cdot)$ (which would require computing and averaging $n$ individual gradients). This key observation facilitates the analysis of SGD, while simultaneously explaining why SGD is much more efficient than gradient descent on large-scale problems.

However, the practice of SGD differs somewhat from this idealized description: In practice, it is much more common to perform without-replacement sampling of the indices, by randomly shuffling the indices $\{1,\ldots,n\}$ and processing the individual functions in that order -- that is, choosing a random permutation $\sigma$ on $\{1,\ldots,n\}$, and performing updates of the form $\bx_{t+1}=\bx_t-\eta_t \nabla f_{\sigma(t)}(\bx_t)$. After a full pass over the $n$ functions, further passes are made either with the same permutation $\sigma$ (known as \emph{single shuffling}), or with a new random permutation $\sigma$ chosen before each pass (known as \emph{random reshuffling}). Such without-replacement schemes are not only more convenient to implement in many cases, they often also exhibit faster error decay than with-replacement SGD \citep{bottou2009curiously,recht2012beneath}. Unfortunately, analyzing this phenomenon has proven to be notoriously difficult. This is because without-replacement sampling creates statistical dependencies between the iterations, so the stochastic gradients computed at each iteration can no longer be seen as unbiased estimates of gradients of $F(\cdot)$. 

Nevertheless, in the past few years our understanding of this problem has significantly improved. For concreteness, let us focus on a classical setting where each $f_i(\cdot)$ is a convex quadratic, and $F(\cdot)$ is strongly convex. Suppose that SGD is allowed to perform $k$ epochs, each of which consists of passing over the $n$ individual functions in a random order. Assuming the step size is appropriately chosen, it has been established that SGD with single-shuffling returns a point whose expected optimization error is on the order of $1/(nk^2)$. For random reshuffling, this further improves to $1/(nk)^2 + 1/(nk^3)$ (see further discussion in the related work section below). In contrast, the expected optimization error of \emph{with}-replacement SGD after the same overall number of iterations ($nk$) is well-known to be on the order of $1/(nk)$ (e.g., \citet{nemirovski2009robust,rakhlin2012making}). Moreover, these bounds are known to be unimprovable in general, due to the existence of nearly matching lower bounds \citep{safran2020good,rajput2020closing}. Thus, as the number of epochs $k$ increases, without-replacement SGD seems to provably beat with-replacement SGD. 

\begin{table}
	\caption{Upper bounds on the expected optimization error for quadratic strongly-convex problems, ignoring constants and log factors. The upper table corresponds to random reshuffling, and the lower table corresponds to single shuffling. The right column presents a necessary (possibly non-sufficient) condition for the bound to be valid and smaller than that of with-replacement SGD, which is order of $1/(\lambda nk)$ \citep{shamir2013stochastic,jain2019making}. An asterisk (*) denotes results where the bound is on the expected squared distance to the global minimum, rather than optimization error. For squared distance, the corresponding bound for with-replacement SGD is order of $1/(\lambda^2 nk)$ \citep{nemirovski2009robust,rakhlin2012making}.}
	\label{table:bounds}
	\begin{center}
		\bgroup
		\def\arraystretch{1.5}
		\begin{tabular}{ccc}
			\hline
			Paper & Bound & Improves on with-Replacement?\\\hline\hline
			\cite{gurbuzbalaban2015random}(*) & 
			$\frac{1}{(\lambda k)^2}$ & Only if $k\gtrsim n$ \\\hline
			\citet{haochen2018random} & $\frac{1}{\lambda^4}\left(\frac{1}{(nk)^2}+\frac{1}{k^3}\right)~~\text{if}~~ k\gtrsim \frac{1}{\lambda}$
			& Only if $k\gtrsim \frac{1}{\lambda}\cdot \max\{1,\sqrt{\frac{n}{\lambda}}\}$\\\hline
			\citet{rajput2020closing} & $\frac{1}{\lambda^4}\left(\frac{1}{(nk)^2}+\frac{1}{nk^3}\right)~~\text{if}~~ k\gtrsim \frac{1}{\lambda^2}$ & Only if $k\gtrsim \frac{1}{\lambda^2}$\\\hline
			\citet{ahn2020sgd} &
			$\frac{1}{\lambda^4}\left(\frac{1}{(nk)^2}+\frac{1}{\lambda^2 nk^3}\right)$ & Only if $k\gtrsim  \frac{1}{\lambda^{2.5}}$\\\hline
		\end{tabular}
		\egroup
	\end{center}
	\vskip 0.5cm
	\begin{center}
	\def\arraystretch{1.5}
	\begin{tabular}{ccc}
	\hline
	Paper & Bound & Improves on with-Replacement?\\\hline\hline
	\cite{gurbuzbalaban2015convergence}(*) & $\frac{1}{(\lambda k)^2}$ & Only if $k\gtrsim n$\\\hline
		\citet{ahn2020sgd} &
	$\frac{1}{\lambda^4 nk^2}~~~\text{if}~~~ k\gtrsim \frac{1}{\lambda^2}$& Only If $k\gtrsim  \frac{1}{\lambda^{2}}$\\\hline
	\cite{mishchenko2020random}(*) & $\frac{1}{\lambda^3 nk^2}$ & Only if $k\gtrsim \frac{1}{\lambda}$\\\hline
	\end{tabular}
	\end{center}
	
\end{table}

Despite these encouraging results, it is important to note that the bounds as stated above quantify only the dependence on $n,k$, and ignore dependencies on other problem parameters. In particular, it is well-known that the convergence of SGD is highly sensitive to the problem's strong convexity and smoothness parameters, their ratio (a.k.a.\ the condition number), as well as the magnitude of the gradients. Unfortunately, existing lower bounds ignore some of these parameters (treating them as constants), while in known upper bounds these can lead to vacuous results in realistic regimes. To give a concrete example, let us return to the case where each $f_i(\cdot)$ is a convex quadratic function, and assume that $F(\cdot)$ is $\lambda$-strongly convex for some $\lambda>0$. In Table~\ref{table:bounds}, we present existing upper bounds for this case, as a function of both $n,k$ as well as $\lambda$ (fixing other problem parameters, such as the smoothness parameter, as constants).\footnote{Focusing only on $\lambda$ is enough for the purpose of our discussion here, and moreover, the dependence on other problem parameters is not always explicitly given in the results of previous papers.} The important thing to note about these bounds is that they are valid and improve over the bound for with-replacement SGD only when the number of epochs $k$ is at least $n$ or $1/\lambda$ (or even more). Unfortunately, such a regime is problematic for two reasons: First, large-scale high-dimensional problems are often ill-conditioned, with $\lambda$ being very small and $n$ being very large, so requiring that many passes $k$ over all functions can easily be prohibitively expensive. Second, if we allow $k> \frac{1}{\lambda}$, there exist \emph{much} better and simpler methods than SGD: Indeed, we can simply run deterministic gradient descent for $k$ iterations (computing the gradient of $F(\cdot)$ at each iteration by computing and averaging the gradients of $f_i(\cdot)$ in any order we please). Since the optimization error of gradient descent (as a function of $k,\lambda$) scales as $\exp(-\lambda k)$ \citep{nesterov2018lectures}, we compute a nearly exact optimum (up to error $\epsilon$) as long as $k\gtrsim \frac{1}{\lambda}\log(\frac{1}{\epsilon})$. Moreover, slightly more sophisticated methods such as accelerated gradient descent or the conjugate gradient method enjoy an error bound scaling as $\exp(-k\sqrt{\lambda})$, so in fact, we can compute an $\epsilon$-optimal point already when $k\gtrsim \frac{1}{\sqrt{\lambda}}\log(\frac{1}{\epsilon})$. 

Thus, the power of SGD is mostly when $k$ is relatively small, and definitely smaller than quantities such as $1/\lambda$. However, the upper bounds discussed earlier do not imply any advantage of without-replacement sampling in this regime. Of course, these are only upper bounds, which might possibly be loose. Thus, it is not clear if this issue is simply an artifact of the existing analyses, or a true issue of without-replacement sampling methods. 

In this paper, we rigorously study this question, with the following (perhaps surprising) conclusion: At least in the worst-case, without-replacement schemes \emph{do not} significantly improve over with-replacement sampling, unless the number of epochs $k$ is larger than the problem's condition number (which suitably defined, scales as $1/\lambda$). As discussed above, this implies that the expected benefit of without-replacement schemes may not be manifest for realistic iteration budgets. In more detail, our contributions are as follows:
\begin{itemize}
	\item We prove that there is a simple quadratic function $F(\cdot)$, which is $\lambda$-strongly convex and $\lambda_{\max}$-smooth, such that the expected optimization error of SGD with single shuffling (using any fixed step size) is at least
	\[
	\Omega\left(\frac{1}{\lambda nk}\cdot \min\left\{1,\frac{\lambda_{\max}/\lambda}{k}\right\}\right)~.
	\]
	(See \thmref{thm:ss_lower_bound}.) Comparing this to the $\Theta(1/(\lambda nk))$ bound for with-replacement SGD, we see that we cannot possibly get a significant improvement unless $k$ is larger than the condition number term $\frac{\lambda_{\max}}{\lambda}$. 
	\item For SGD with random reshuffling, we prove a similar lower bound of
	\[
	 \Omega\left(\frac{1}{\lambda nk}\cdot\min\set{1~,~\frac{\lambdamax/\lambda}{nk} + \frac{\lambdamax^2/\lambda^2}{k^2}}\right)~.
	\]
	(See \thmref{thm:rr_lower_bound}.) As before, this improves on the $\Theta(1/(\lambda nk))$ bound of with-replacement SGD only when $k>\frac{\lambda_{\max}}{\lambda}$. 
	\item We provide matching upper bounds in all relevant parameters (up to constants and log factors and assuming the input dimension is fixed -- See \secref{sec:upper}), which apply to the class of quadratic functions with commuting quadratic terms (which include in particular the lower bound constructions above). 
	\item To illustrate our theoretical results, we perform a few simple experiments comparing with- and without-replacement sampling schemes on our lower bound constructions (see \secref{sec:exp}). Our results accord with our theoretical findings, and show that if the number of epochs is not greater than the condition number, then without-replacement does not necessarily improve upon with-replacement sampling. Moreover, it is observed that for some of the values of $k$ exceeding the condition number by as much as $50\%$, with-replacement still provides comparable results to without-replacement, even when averaging their performance over many instantiations.
\end{itemize}

We conclude with a discussion of the results and open questions in \secref{sec:discussion}.

We note that our lower and upper bounds apply to SGD using a fixed step size, $\eta_t=\eta$ for all $t$, and for the iterate reached after $k$ epochs. Thus, they do not exclude the possibility that better upper bounds can be obtained with a variable step size strategy, or some iterate averaging scheme. However, we conjecture that it is not true, as existing upper bounds either do not make such assumptions or do not beat the lower bounds presented here. Moreover, SGD with variable step sizes can generally be matched by SGD employing a fixed optimal step size (dependent on the problem parameters and the overall number of iterations).

%

\subsection*{Related Work}

Proving that without-replacement SGD converges faster than with-replacement SGD has been the focus of a line of recent works, which we briefly survey below (focusing for concreteness on quadratic and strongly convex problems, as we do in this paper). \citet{gurbuzbalaban2015convergence,gurbuzbalaban2015random} proved that without-replacement beats with-replacement in terms of the dependence on $k$ ($1/k^2$ vs. $1/k$, without specifying a decay in terms of $n$). \cite{shamir2016without} proved that without-replacement is not worse than with-replacement in terms of dependence on $n$, but only for $k=1$. \citet{haochen2018random} managed to prove a better bound for random reshuffling, scaling as $1/(nk)^2+1/k^3$. \citet{safran2020good} proved a lower bound of $1/(nk)^2+1/nk^3$ for random reshuffling and $1/(nk^2)$ for single shuffling, and also proved matching upper bounds in the (very) special case of one-dimensional quadratic functions. A series of recent works \citep{jain2019sgd,rajput2020closing,ahn2020sgd,mishchenko2020random,nguyen2020unified} showed that these are indeed the optimal bounds in terms of $n,k$, by proving matching upper bounds which apply to general quadratic functions and beyond. \citet{rajput2020closing} were also able to show that for strongly convex and smooth functions beyond quadratics, the error rate provably becomes order of $1/(nk^2)$ for both single shuffling and random reshuffling.

A related and ongoing line of work considers the question of whether without-replacement SGD can be shown to be superior to with-replacement SGD, individually on any given problem of a certain type \citep{recht2012beneath,lai2020recht,de2020random,yun2021can}. However, this is different than comparing worst-case behavior over problem classes, which is our focus here. Without-replacement SGD was also studied under somewhat different settings than ours, such as \citet{ying2018stochastic,tran2020shuffling,huang2021comparison}. 

In our paper, we focus on the well-known and widely popular SGD algorithm, using various sampling schemes. However, we note that for finite-sum problems, different and sometimes better convergence guarantees can be obtained using other stochastic algorithms, such as variance-reduced methods, adaptive gradient schemes, or by incorporating momentum. Analyzing the effects of without-replacement sampling on such methods is an interesting topic (studied for example in \cite{tran2020shuffling}), which is however outside the scope of our paper. 

\section{Preliminaries}

\textbf{Notation and terminology.} We let bold-face letters (such as $\bx$) denote vectors. For a vector $\bx$, $x_j$ denotes its $j$-th coordinate. For natural $n$, let $[n]$ be shorthand for the set $\set{1,\ldots,n}$. For some vector $\bv$, $\norm{\bv}$ denotes its Euclidean norm. We let $\norm{\cdot}_{\textnormal{sp}}$ denote the spectral norm of a matrix. A twice-differentiable function $f$ on 
$\reals^d$ is $\lambda$-strongly convex, if its Hessian satisfies $\nabla^2 F(\bx)\succeq 
\lambda I$ for all $\bx$, and is $L$-smooth if its gradient is $L$-Lipschitz. 
$f$ is quadratic if it is of the form $f(\bx)=\frac12\bx^\top A\bx+\mathbf{b}^\top 
\bx+c$ for some matrix $A$, vector $\mathbf{b}$ and scalar $c$.\footnote{We note that throughout the paper we omit the constant scalar terms, since these do not affect optimization when performing SGD (only the optimal value attained).} Note that if $A$ is a PSD matrix, then the strong convexity and smoothness parameters of $f$ correspond to the smallest and largest eigenvalues of $A$, respectively. Moreover, the ratio between the largest and smallest eigenvalues is known as the condition number of $f$. We use standard asymptotic notation  $\Theta(\cdot),\Ocal(\cdot),\Omega(\cdot)$ to hide constants, and $\tilde{\Theta}(\cdot),\tilde{\Ocal}(\cdot),\tilde{\Omega}(\cdot)$ to hide constants and logarithmic factors.

\textbf{SGD.} As mentioned in the introduction, we focus on plain SGD using a constant step size $\eta$, which performs $k$ epochs, and in each one takes $n$ stochastic gradients steps w.r.t.\ the functions $f_{\sigma(1)},\ldots,f_{\sigma(n)}$, with $\sigma$ being a random permutation which is either sampled afresh at each epoch (random reshuffling) or chosen once and then used for all $k$ epochs (single shuffling). We let $\bx_0$ denote the initialization point, and $\bx_1,\ldots,\bx_k$ denote the iterates arrived at the end of epoch $1,\ldots,k$ respectively.

\section{Lower Bounds}\label{sec:lower_bounds}

In this section, we formally present our lower bounds for single shuffling and random reshuffling SGD. Our lower bounds will use a particularly simple class of quadratic functions, which satisfy the following assumptions.
\begin{assumption}[Lower Bounds Assumption]\label{as:lower_bounds}
	$F(\bx)$ is $\lambda$-strongly convex and of the form $\frac{1}{n}\sum_{i=1}^{n}f_i(\bx)$, and each $f_i$ is of the form $f_i(\bx):=\sum_{j}\frac12a_j x_j^2-b_j x_j$, where for all $j$, $a_j\in [\lambda,\lambda_{\max}]\cup\{0\}$ and $b_j\in [-\frac{G}{2},\frac{G}{2}]$. Suppose $nk$ is large enough so that $\frac{\log(nk) L}{\lambda nk}\le1$. We assume that the algorithm is initialized at some $\bx_0$ on which $\norm{\nabla F(\bx_0)}\le G$.
\end{assumption}

Note that any such function $F$ is $\lambda_{\max}$-smooth, and so are its components. We emphasize that since our lower bounds apply to such functions, they also automatically apply to larger function classes (e.g.\ more general quadratic functions, the class of all $\lambda$-strongly convex and $\lambda_{\max}$-smooth functions, etc). Moreover, our lower bounds also apply in harder and more general settings where for example only a partial or noisy view of the Hessian of the functions is revealed.

Having stated our assumptions, we now turn to present our results for this section, beginning with our single shuffling lower bound.

\begin{theorem}[Single Shuffling Lower Bound]\label{thm:ss_lower_bound}
	For any $k\ge1, n>1$, and positive $G,\lambda$, there exist a function $F$ on $\reals^2$ and an initialization point $\bx_0$ satisfying Assumption \ref{as:lower_bounds}, for which single shuffling SGD using any fixed step size $\eta>0$ satisfies
	\[
		~\E\pcc{F(\bx_k)} ~\ge~ c\frac{G^2}{\lambda nk}\min\set{1~,~ \frac{\lambdamax/\lambda}{k}}~,
	\]
	for some universal constant $c>0$.
\end{theorem}

The function $F$ used in the above theorem is given by the component functions
\begin{equation}\label{eq:ss_construction}
	~f_{i}(\bx)~=~f_i(x_1,x_2)~:=~\frac{\lambda}{2}x_1^2+\frac{\lambdamax}{2}x_2^2+\begin{cases}\frac{G}{2}x_2& i\leq \frac{n}{2}\\-\frac{G}{2}x_2& i>\frac{n}{2}\end{cases}~.
\end{equation}
For the random reshuffling sampling scheme, we have the following theorem.

\begin{theorem}[Random Reshuffling Lower Bound]\label{thm:rr_lower_bound}
	For any $k\ge1, n>1$, and positive $G,\lambda$, there exist a function $F$ on $\reals^3$ and an initialization point $\bx_0$ satisfying Assumption \ref{as:lower_bounds}, for which random reshuffling SGD using any fixed step size $\eta>0$ satisfies
	\[
		~\E\pcc{F(\bx_k)} ~\ge~ c\frac{G^2}{\lambda nk}\min\set{1~,~\frac{\lambdamax/\lambda}{nk} + 		\frac{\lambdamax^2/\lambda^2}{k^2}}~,
	\]
	for some universal constant $c>0$.
\end{theorem}

The function $F$ used in the above theorem is given by the component functions
\begin{equation}\label{eq:rr_construction}
	~f_{i}(\bx)~=~f_i(x_1,x_2,x_3)~\coloneqq~\frac{\lambda}{2}x_1^2+\frac{\lambdamax}{2}x_2^2 + \begin{cases}\frac{G}{2}x_2 + \frac{\lambdamax}{2}x_3^2 + \frac{G}{2}x_3& i\leq \frac{n}{2}\\-\frac{G}{2}x_2 - \frac{G}{2}x_3& i>\frac{n}{2}\end{cases}~.
\end{equation}

In contrast, with-replacement SGD bounds are of the form $\tilde{\Ocal}(G^2/(\lambda nk))$.\footnote{There is some subtlety here, as most such bounds are either on some average of the iterates rather than the last iterate, or use a non-constant step-size (in which case it is possible to prove a bound of $\Ocal(G^2/\lambda nk)$). However, \cite{shamir2013stochastic} prove a bound of $\tilde{\Ocal}(G^2/\lambda nk)$ (involving an additional $\log(nk)$ factor) for the last iterate, using a variable step-size $\eta_t = 1/(\lambda t)$, and it is not difficult to adapt the analysis to SGD with a constant step size $\eta=\tilde{\Theta}(1/(\lambda nk))$.} Thus, we see that for both single shuffling and random shuffling, we cannot improve on the worst-case performance of with-replacement SGD by more than constants or logarithmic terms, unless $k$ is larger than the condition number parameter $\frac{\lambda_{\max}}{\lambda}$. As discussed earlier, this can be an unrealistically large regime for $k$ in many cases.
	
The proofs of the above theorems, which appear in Appendices \ref{subsec:ss_lower_proof} and \ref{subsec:rr_lower_proof}, follow a broadly similar strategy to the one-dimensional constructions in \citet{safran2020good}, with the crucial difference that we construct a multivariate function having different curvature in each dimension. The constructions themselves are very simple, 
but the analysis is somewhat technical and intricate. Roughly speaking, our analysis splits into three different cases, where in each a different magnitude of the step size is considered. In the first case, where the step size is too small, we use a dimension with small curvature to show that if we initialize far enough from the global minimizer then SGD will converge too slowly. The construction in the remaining dimensions is similar to that in \citet{safran2020good}, and uses a larger curvature than the first dimension to generate dependency on the condition number, due to the somewhat too large step size that is dictated by the first case. In the remaining two cases, the step size is either too large and the resulting bound is not better than with-replacement SGD; or is chosen in a manner which achieves better dependence on $n,k$, but also incurs additional error terms dependent on $\lambda_{\max},\lambda$, due to the statistical dependencies between the gradients created via without-replacement sampling. Combining the cases leads to the bounds stated in the theorems. 


We also make the following remarks regarding our lower bounds:
\begin{remark}[Separating $\lambdamax$ and $L$]
	Our parameters satisfy the following chain of inequalities $\lambda\le\lambdamax\le L$. Since the condition number is commonly defined in the literature as $L/\lambda$, we note that in our lower bound constructions we have $\lambdamax=L$. However, since $\lambdamax$ could potentially be smaller than $L$, this raises the question of whether it is possible to construct a lower bound in which $\lambdamax$ is sufficiently smaller than $L$ yet the lower bound depends on $L$. The upper bounds we present in the next section indicate that this is not possible when the $A_i$'s commute, since in this case we get bounds with no dependence on $L$ (at least when $nk$ is large enough). This implies that to separate the two quantities $\lambdamax,L$, one would necessarily need a construction where the $A_i$'s do not commute. We leave the derivation of such a construction to future work.
\end{remark}

\begin{remark}[Bound on $L$]\label{remark:nk}
	We note that the condition $\frac{\log(nk) L}{\lambda nk}\leq 1$ is equivalent to requiring $nk$ to be at least on the order of the condition number $L/\lambda$ (up to log factors). This is a mild requirement, since if $nk$ is smaller than the condition number, then even if $f_i=f$ for all $i$ (that is, we perform deterministic gradient descent on the function $f$), there are quadratic functions for which no non-trivial guarantee can be obtained \citep{nesterov2018lectures}.
\end{remark}

\begin{remark}[A Random Reshuffling Lower Bound in $\reals^2$]
	Our lower bound construction for single shuffling is a function in $\reals^2$, whereas our random reshuffling construction is in $\reals^3$. We believe that a construction in $\reals^2$ is also possible for random reshuffling, however this would require a more technically involved proof which for example also lower bounds the right-most summand in \eqref{eq:ignored_squared_term}.
\end{remark}

\section{Matching Upper Bounds}\label{sec:upper}

In this section, we provide upper bounds that match our lower bounds from the previous section (up to constants and log factors, and for a fixed input dimension). Whereas our lower bounds use quadratic constructions where each matrix $A_i$ is diagonal, here we prove matching upper bounds in the somewhat more general case where the quadratic terms commute. Before we state our upper bounds, however, we will first state and discuss the assumptions used in their derivation.

\begin{assumption}[Upper Bounds Assumption]\label{as:upper_bounds}
	Suppose the input dimension $d$ is a fixed constant. Assume $F(\bx)=\frac12\bx^{\top}A\bx - \bb^{\top}\bx$ is a quadratic finite-sum function of the form 
	$\frac{1}{n}\sum_{i=1}^{n}f_i(\bx)$ for some $n>1$, which is 
	$\lambda$-strongly convex and satisfies $\norm{A}_{\textnormal{sp}} = \lambdamax$. 
	Each $f_i$ is a convex quadratic function of the form $f_i(x)=\frac12\bx^{\top}A_i\bx - \bb_i^{\top}\bx$, for commuting, symmetric, and PSD matrices $A_i\in\reals^{d\times d}$, which all satisfy $\norm{A_i}_{\textnormal{sp}}\leq L$. Moreover, suppose that $\norm{\nabla f_i(\bx^*)}\le G$ for all $i\in[n]$, where $\bx^*\coloneqq\arg\min_{\bx}F(\bx)$ is the global minimizer. Assume $nk$ is large enough so that $\frac{\log(nk) L}{\lambda nk}\le1$. We assume that the algorithm is initialized at some $\bx_0$ on which $\norm{\nabla F(\bx_0)}\le G$.
\end{assumption}

Note that our lower bound constructions satisfy the assumptions above. We remark that as far as upper bounds go, these assumptions are somewhat different than those often seen in the literature. E.g.\ we require that $\nabla f_i$ is bounded only at the global minimizer $\bx^*$, rather than in some larger domain as commonly assumed in the literature. However this comes at a cost of assuming that $d$ is fixed. We remark that we believe that it is possible to extend our upper bounds to hold for arbitrary $d$ with no dependence on $d$, but this would require a more technically involved analysis (e.g.\ a multivariate version of \propref{prop:random_variable_bounds}). Nevertheless, the bounds presented here are sufficient for the purpose of establishing the tightness of our lower bound constructions, which involve 2-3 dimensions. Lastly, we recall that the assumption on the magnitude of $nk$ is very mild as explained in Remark \ref{remark:nk}.

We now present the upper bound for single shuffling SGD:
\begin{theorem}[Single Shuffling Upper Bound]\label{thm:ss_upper_bound_commuting}
	Suppose $F(\bx)=\frac1n\sum_{i=1}^{n}f_i(\bx)$ satisfy Assumption \ref{as:upper_bounds}, and fix the step size $\eta=\frac{\log(nk)}{\lambda nk}$. Then single shuffling SGD satisfies
	\[
		~\E\pcc{F(\bx_k)}~\leq~\tilde{\Ocal}\p{\frac{G^2}{\lambda n k}\cdot\min\set{1 ~,~ \frac{\lambdamax/\lambda}{k}}}~,
	\]
	where the $\tilde{\Ocal}$ hides a universal constant, factors logarithmic in $n,k,\lambdamax,1/\lambda$, and a factor linear in $d$.
\end{theorem}




Next, we have the following result for random reshuffling SGD.

\begin{theorem}[Random Reshuffling Upper Bound]\label{thm:rr_upper_bound_commuting}
	Suppose $F(\bx)=\frac1n\sum_{i=1}^{n}f_i(\bx)$ satisfy Assumption \ref{as:upper_bounds}, and fix the step size $\eta=\frac{\log(nk)}{\lambda nk}$. Then random reshuffling SGD satisfies
	\[
	~\E\pcc{F(\bx_k)}~\leq~\tilde{\Ocal}\p{\frac{G^2}{\lambda n k}\cdot\min\set{1 ~,~ \frac{\lambdamax/\lambda}{nk} + \frac{\lambdamax^2/\lambda^2}{k^2}}}~,
	\]
	where the $\tilde{\Ocal}$ hides a universal constant, factors logarithmic in $n,k,\lambdamax,1/\lambda$, and a factor linear in $d$.
\end{theorem}

The proofs of the above theorems, which appear in Appendices \ref{subsec:ss_upper_proof} and \ref{subsec:rr_upper_proof}, are based on reducing our problem to a one-dimensional setting, deriving a closed-form expression for $x_k$ (as was done in \citet{safran2020good}), and then carefully bounding each term in the resulting expression. However, our bounds refine those of \citet{safran2020good}, allowing us to achieve upper bounds that precisely capture the dependence on all the problem parameters. Both proofs split the analysis into two cases: (i) The case where we perform sufficiently many epochs ($k>\lambdamax/\lambda$) and are able to improve upon with-replacement SGD; and (ii), where too few epochs are performed ($k<\lambdamax/\lambda$), in which case the sub-optimality rate matches that of with-replacement SGD. In both theorems, new tools to control the random quantities that arise during the problem analysis are required to attain bounds that match our lower bounds (see \propref{prop:random_variable_bounds} in the Appendix for example). We point out that the choice of step size in the theorems is not unique and that the same upper bounds, up to log factors, apply to other choices of the step size (we did not attempt to optimize the logarithmic terms in the bounds). 

It is interesting to note that somewhat surprisingly, our upper bounds do not depend on $L$, only on the smaller quantity $\lambdamax$. $L$ only appears via the condition $\eta L\leq 1$ in Assumption \ref{as:upper_bounds}. This can be attributed to the assumption that the $A_i$'s commute, in which case the spectral norm of their average governs the bound rather than the maximal norm of the individual $A_i$'s. It is interesting to investigate to what extent our bounds could generalize to the non-commuting case, and if the dependence on $L$ will remain mild. Since our upper bounds make essential use of the AM-GM inequality and of the even more general Maclaurin's inequalities, it seems like we would need strong non-commutative versions of these inequalities to achieve such bounds using our technique. We refer the interested reader to \citet{yun2021can} for further discussion on non-commutative versions of such inequalities.

Another noteworthy observation is regarding the phase transitions in the bounds. In the previously known upper bounds for random reshuffling for quadratics, which scale as $\frac{1}{(nk)^2}+\frac{1}{nk^3}$ (see Table~\ref{table:bounds}), the dominant term switches from $\frac{1}{nk^3}$ to $\frac{1}{(nk)^2}$ once $k\geq n$. In the ill-conditioned setting, our random reshuffling upper bound reveals that such a phase transition only occurs when $k>\frac{\lambdamax}{\lambda}n$, which is a significantly stronger assumption. This suggests that in many real-world applications, the sub-optimality rate would be $\frac{G^2}{\lambda nk}\cdot\min\set{1~,~\frac{\lambdamax^2/\lambda^2}{k^2}}$, unless one can afford a very large iteration budget.

Lastly, we make the following additional remarks on our upper bounds.





\begin{remark}[Bounds with High Probability]
	Our proof for the single shuffling upper bound can be adapted to imply a high-probability bound, rather than merely an in-expectation bound, at a cost of a logarithmic dimension dependence (e.g.\ by taking a union bound over the guarantee in \thmref{thm:ss_upper_bound} in the proof across all dimensions). Moreover, in the case where $k\le \lambdamax/\lambda$, then our random reshuffling upper bound also holds with high probability.\footnote{This is because that in this case, the dominant term in the high probability bound given by $\frac{\lambdamax/\lambda}{k}+\frac{\lambdamax^2/\lambda^2}{k^2}$ is the squared term, which is also the dominant term in the upper bound in \thmref{thm:rr_upper_bound_commuting}.} This implies a concentration of measure around the expected bounds, which shows that our derived sub-optimality rates are the 'typical' behavior of without-replacement SGD in the worst case for single shuffling, and that the inability to improve upon without replacement when $k\le \lambdamax/\lambda$ is not just on average but also with high probability for random reshuffling.
\end{remark}

\begin{remark}[The Problem is Always Well-Conditioned in One-Dimension]
	Since in the one-dimensional case we have $\lambda=\lambdamax$, our upper bounds reveal that perhaps in contrast to what previous upper bounds that depended on $L$ suggested, the one-dimensional case is always well-conditioned. That is, without-replacement will always beat with-replacement SGD in one-dimension, given that Assumption \ref{as:upper_bounds} holds.
\end{remark}

\section{Experiments}\label{sec:exp}
	In this section, we empirically verify the phenomenon that our theory predicts by running simulations on the constructions used in Equations (\ref{eq:ss_construction}) and (\ref{eq:rr_construction}).\footnote{For the sake of simplicity and to better distinguish the two constructions, we only used the first and third dimensions in the construction in \eqref{eq:rr_construction}, since the second dimension is identical to the second dimension in \eqref{eq:ss_construction}.} Our lower bounds show that to beat with-replacement on those problem instances, the number of epochs must be in the order of magnitude of the condition number. However, since our analysis is asymptotic in its nature, ignoring constants and logarithmic terms, it is not clear to what extent this phenomenon can be observed in practice. To investigate this, we averaged the performance of $100$ SGD instantiations over various values of $k$ ranging from $k=40$ up to $k=2,000$, where for each value a suitable step size of $\eta=\frac{\log(nk)}{\lambda nk}$ was chosen. Our problem parameters were chosen to satisfy $n=500$, $G=1$, $\lambda=1$ and $\lambdamax=200$. Note that this implies that the condition number equals $200$ in the constructed problem instances. Each SGD instantiation was initialized from $\bx_0=\p{-\frac{G}{2\lambda},-\frac{G}{2\lambdamax}}$, which satisfies both Assumptions \ref{as:lower_bounds} and \ref{as:upper_bounds}. We used Python 3.6 in our code, which is freely available at https://github.com/ItaySafran/SGD\_condition\_number. Our code was run on a single machine with an Intel Core i7-8550U 1.80GHz processor.
	
	Our experiment results, which appear in \figref{fig:exp}, reveal that to get a noticeable advantage over with-replacement, the number of epochs must be very large and indeed exceed the condition number. Moreover, by that point the remaining error is already extremely tiny (less than $10^{-5}$). Additionally, for values of $k$ up to $300$, it is evident that with-replacement sampling occasionally performed almost the same as without-replacement sampling schemes, exhibiting significant overlap of the confidence intervals. This indicates that in certain ill-conditioned applications, it can be quite common for with-replacement to beat without-replacement. Another interesting observation is that random reshuffling does not significantly improve upon single shuffling, unless $k$ is at least $1,000$. This is in line with our theoretical results, which indicate that the advantage of random reshuffling will only be manifest once $k$ is considerably larger than the condition number. It also indicates that the additional effort of reshuffling the functions in every epoch might not pay off on ill-conditioned problems, unless a considerable iteration budget is possible.
	
	We remark that a similar experiment was shown in \citet{de2020random}, where a construction is presented in which with-replacement outperforms without-replacement SGD. However, a major difference between the above experiment and ours is that in the former, despite under-performing compared to with-replacement, without-replacement achieves an extremely small error, which requires exact arithmetic to monitor, to the point where the difference becomes arguably insignificant from a practical perspective. In contrast, our experiment shows that without-replacement merely does not significantly improve upon with-replacement for moderate values of $k$, in a setting where much more realistic error rates are attained.
	
	\begin{figure}
		\begin{center}
			\includegraphics[width=0.83\textwidth,trim=0.7cm 0cm 0.7cm 1cm,clip]{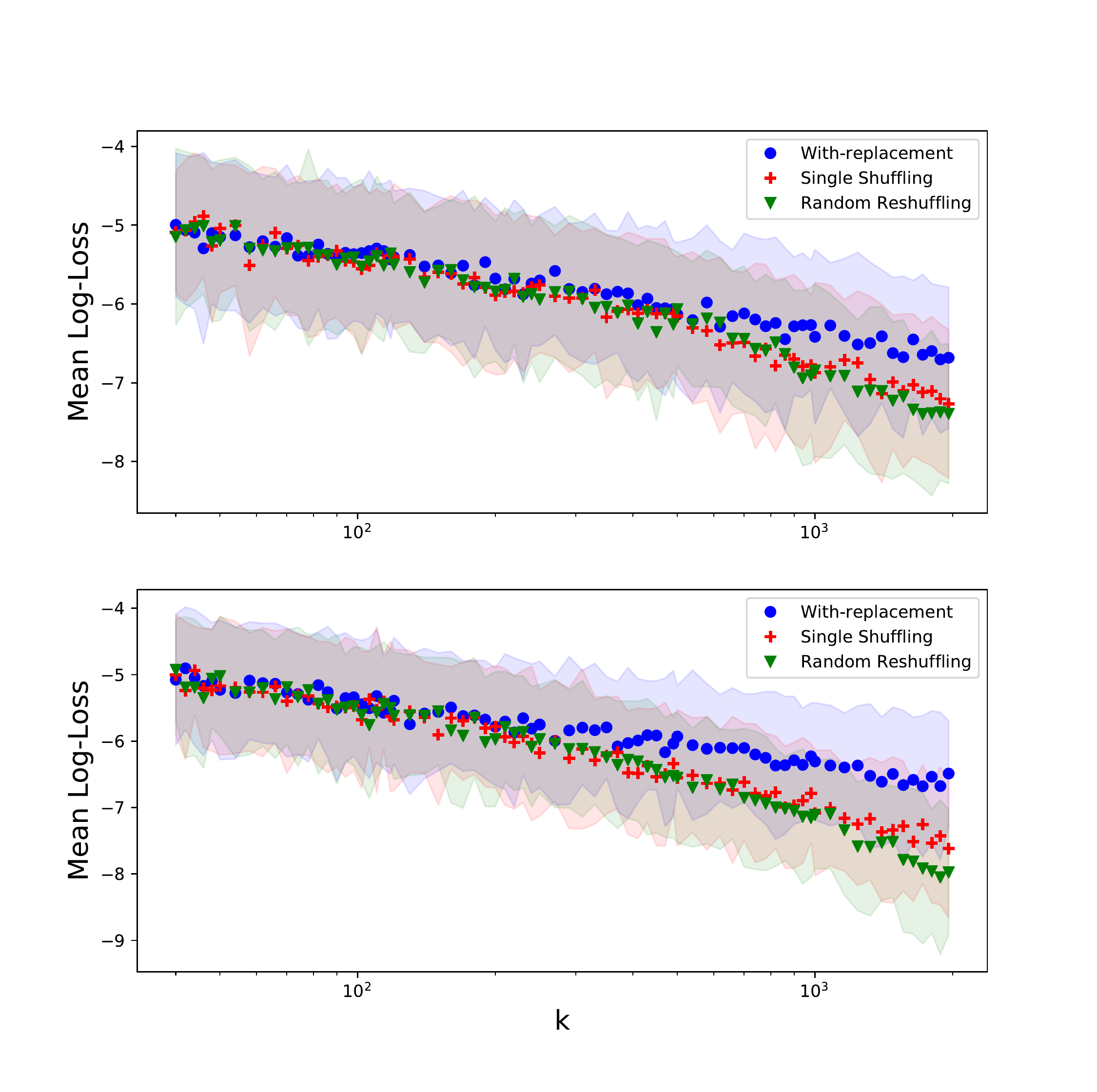}
		\end{center}
		\caption{The average value of $\log_{10}F(\bx_k)$ attained after running $100$ instantiations of SGD on $F(\cdot)$ in \eqref{eq:ss_construction} (top) and \eqref{eq:rr_construction} (bottom), using with-replacement (blue circle), single shuffling (red plus) and random reshuffling (green triangle) sampling, for varying values of $k$ and where the step size is chosen accordingly. The confidence intervals depict a single standard deviation of the log-loss for each value of $k$. Best viewed in color.}\label{fig:exp}
	\end{figure}

\section{Discussion}\label{sec:discussion}

In this paper, we theoretically studied the performance of without-replacement sampling schemes for SGD, compared to the performance of with-replacement SGD. Perhaps surprisingly (in light of previous work on this topic), we showed that without-replacement schemes \emph{do not} significantly improve over with-replacement sampling, unless the number of epochs $k$ is larger than the condition number, which is unrealistically large in many cases. Although the results are in terms of worst-case bounds, they already hold over a class of rather simple and innocuous quadratic functions. We also showed upper bounds essentially matching our lower bounds, as well as some simple experiments corroborating our findings.

Overall, our results show that the question of how without-replacement schemes compare to with-replacement sampling is intricate, and does not boil down merely to convergence rates in terms of overall number of iterations. Moreover, it re-opens the question of when we can hope without-replacement schemes to have better performance, for moderate values of $k$. Our lower bounds imply that for ill-conditioned problems, this is not possible in the worst-case, already for very simple quadratic constructions. However, an important property of these constructions (as evidenced in Equations (\ref{eq:ss_construction}) and (\ref{eq:rr_construction})) is that the eigenvalue distribution of the matrices defining the quadratic terms is sharply split between $\lambda_{\max},\lambda$, and both directions of maximal and minimal curvature are significant. In a sense, this makes it impossible to choose a step size that is optimal for all directions, and $k$ is indeed required to be larger than the condition number $\lambda_{\max}/\lambda$ for without-replacement schemes to beat with-replacement sampling. In contrast, for quadratics whose  eigenvalue distribution is smoother (e.g.\ where the curvature is roughly the same in ``most'' directions), we expect the ``effective'' condition number to be smaller, and hence possibly see an improvement already for smaller values of $k$. Formally quantifying this, and finding other favorable cases for without-replacement schemes, is an interesting question for future work. Another open question is to generalize our upper and lower bounds (getting the precise dependence on the condition number) to more general quadratic functions, or the even larger class of strongly-convex and smooth functions.

\section*{Acknowledgments}
	This research is supported in part by European Research Council (ERC) Grant 754705. We thank Gilad Yehudai for helpful discussions.

\bibliographystyle{plainnat}
\bibliography{bib}

\appendix

\section{Proofs}\label{app:proofs}

\subsection{Proof of \thmref{thm:ss_lower_bound}}\label{subsec:ss_lower_proof}

\begin{proof}[\unskip\nopunct] 
	
	
	Assume $n$ is even (this is without loss of generality as argued in the beginning of the proof of Thm.~1 in \citet{safran2020good}). Recall the function $F$ defined in \eqref{eq:ss_construction} by
	\[
		~F(\bx)~=~\frac{1}{n}\sum_{i=1}^{n}f_i(\bx) ~=~ \frac{\lambda}{2}x_1^2 + \frac{\lambdamax}{2}x_2^2~,
	\]
	where for each $i$,
	\begin{equation*}
		~f_{i}(\bx)~=~f_i(x_1,x_2)~:=~\frac{\lambda}{2}x_1^2+\frac{\lambdamax}{2}x_2^2+\begin{cases}\frac{G}{2}x_2& i\leq \frac{n}{2}\\-\frac{G}{2}x_2& i>\frac{n}{2}\end{cases}~.
	\end{equation*}
	It is readily seen that the above functions satisfy Assumption \ref{as:lower_bounds}. Assume we initialize at 
	\[
		~\bx_0 = (x_{0,1}, x_{0,2}) = \p{\frac{G}{\lambda}~, ~0}~,
	\] which also satisfies Assumption \ref{as:lower_bounds} since $\norm{\nabla F(\bx_0)}=G$. On these functions, we have that during any 
	single epoch, we 
	perform $n$ iterations of the form
	\[
	x_{new,1}=(1-\eta\lambda)x_{old,1}~~~,~~~
	x_{new,2}=(1-\eta \lambdamax)x_{old,2}+\frac{\eta G}{2}\sigma_i~,
	\]
	where $\sigma_0,\ldots,\sigma_{n-1}$ are a random permutation of $\frac{n}{2}$ 
	$1$'s and $\frac{n}{2}$ $-1$'s. Repeatedly applying this inequality, we get 
	that after $n$ iterations, the relationship between the first and last iterates 
	in the epoch satisfy
	\[
	x_{t+1,1}~=~(1-\eta\lambda)^n x_{t,1}~~~,~~~
	x_{t+1,2}~=~(1-\eta \lambdamax)^n x_{t,2}+\frac{\eta 
		G}{2}\sum_{i=0}^{n-1}\sigma_i(1-\eta \lambdamax)^{n-i-1}~.
	\]
	Repeating this across $k$ epochs, we obtain the following relation between the initialization point and what we obtain after $k$ epochs:
	\[
	x_{k,1} = (1-\eta\lambda)^{nk}x_{0,1}~~~,~~~
	x_{k,2}~=~
	(1-\eta \lambdamax)^{nk}x_{0,2}+\frac{\eta G}{2}\cdot\frac{1-(1-\eta \lambdamax)^{nk}}{1-(1-\eta \lambdamax)^n}\sum_{i=0}^{n-1}\sigma_i(1-\eta \lambdamax)^{n-i-1}~.
	\]
	Noting that $F(\bx)=\frac{\lambda}{2}x_1^2+\frac{\lambdamax}{2}x_2^2$ and $\E[\sigma_i]=0$, we get that
	\[
	~\E[F(\bx_k)]~=~\frac{\lambda}{2}(1-\eta\lambda)^{2nk}x_{0,1}^2+\frac{\lambdamax}{2}(1-\eta \lambdamax)^{2nk}x_{0,2}^2
	+
	\frac{\eta^2 G^2 \lambdamax}{8}\left(\frac{1-(1-\eta \lambdamax)^{nk}}{1-(1-\eta \lambdamax)^n}\right)^2\beta_{n,\eta,\lambdamax} ~,
	\]
	where 
	\begin{equation}\label{eq:beta}
	~\beta_{n,\eta,\lambdamax}~:=~
	\E\left[\left(\sum_{i=0}^{n-1}\sigma_i(1-\lambdamax\eta)^{n-i-1}\right)^2\right]~=~
	\E\left[\left(\sum_{i=0}^{n-1}\sigma_i(1-\lambdamax\eta)^{i}\right)^2\right]
	\end{equation}
	(using the fact that $\sigma_0,\ldots,\sigma_{n-1}$ are exchangeable random variables). 
	According to Lemma 1 in \citet{safran2020good}, for some numerical constant $c>0$,
	\begin{equation}\label{eq:betalb}
	~\beta_{n,\eta,\lambdamax}~\geq~
	c\cdot \min\left\{1+\frac{1}{\lambdamax\eta}~,~n^3(\lambdamax\eta)^2\right\}~.
	\end{equation}
	We now perform a case analysis based on the value of $\eta$:
	\begin{itemize}[leftmargin=*]
		\item 
		If $ \eta\le\frac{1}{\lambda nk} $, then
		\[
		~\E[F(\bx_k)]~\ge~ \frac{\lambda x_{0,1}^2}{2}(1-\eta\lambda)^{2nk} ~\ge~ 
		\frac{\lambda x_{0,1}^2}{2}\left(1-\frac{1}{nk}\right)^{2nk} ~\ge~ 
		\frac{\lambda x_{0,1}^2}{2}\left(\frac{1}{4}\right)^2~=~\frac{\lambda x_{0,1}^2}{32}~.
		\]
		Substituting $x_{0,1}=G/\lambda$, the above is lower bounded by
		\[
			~\frac{G^2}{32\lambda}~.
		\]
		\item 
		If $ \eta> \frac{1}{\lambda nk}$ as well as $\eta < \frac{1}{\lambdamax n}$ (assuming this range exists, namely when $k>\lambdamax /\lambda$), then by 
		Bernoulli's inequality we have $(1-\eta \lambdamax )^n\ge1-n\eta \lambdamax >0$, as well as $(1-\eta \lambdamax )^{nk}\leq (1-\eta \lambda)^{nk}\leq (1-1/nk)^{nk}\leq \exp(-1)$, implying that
		\[
		\E[F(\bx_k)]~\ge~ \frac{\eta^2 G^2 \lambdamax }{8} \left(\frac{1-\exp(-1)}{1-(1-n\eta \lambdamax )}\right)^2 \beta_{n,\eta,\lambdamax } ~=~ \frac{\eta^2G^2 \lambdamax (1-\exp(-1))^2}{8(n\eta \lambdamax )^2}\cdot\beta_{n,\eta,\lambdamax }~.
		\]
		Plugging in \eqref{eq:betalb}, and noting that $ \eta < \frac{1}{\lambdamax n} 
		$, it is easily verified that $ \beta_{n,\eta,\lambdamax }\ge 
		c\cdot\min\{1+1/\eta \lambdamax ,n^3(\eta \lambdamax )^2\}= cn^3\eta^2 \lambdamax ^2 $. 
		This implies that the displayed equation above is at least
		\[
		~c^{\prime}\frac{\eta^2 G^2 \lambdamax }{(n\eta \lambdamax )^2}\cdot n^3 \eta ^2 \lambdamax ^2 ~=~ c^{\prime}\eta^2nG^2 \lambdamax ~,
		\]
		for some constant $ c^{\prime} $. Since $ \eta> \frac{1}{\lambda nk}$, this is at least
		\[
		~c^{\prime}\frac{nG^2 \lambdamax }{\lambda^2n^2k^2} = c^{\prime}\frac{G^2 \lambdamax }{\lambda^2nk^2}~.
		\] 
		\item 
		If $ \eta> \frac{1}{\lambda nk}$ as well as $\eta \geq \frac{1}{\lambdamax n}$, then 
		noting that $\left(\frac{1-(1-\eta \lambdamax )^{nk}}{1-(1-\eta \lambdamax )^{n}}\right)^2=\left(\sum_{i=0}^{k-1}\left((1-\eta \lambdamax )^{n}\right)^i\right)^2\geq \left((1-\eta \lambdamax )^0\right)^2~=~1$ (recall that $n$ is even), we have
		\[
		~\E[F(\bx_k)]~\geq~ \frac{\eta^2 G^2 \lambdamax }{8}\beta_{n,\eta,\lambdamax }~.
		\]
		By the assumption that $\eta\geq \frac{1}{\lambdamax n}$, we have that $ n^3(\eta \lambdamax )^2\ge 1/\eta \lambdamax $ as well as $n^3(\eta \lambdamax )^2\geq 1$. Using this and \eqref{eq:betalb}, the above is at least
		\[
		\frac{c \eta^2 G^2 \lambdamax }{8}\min\left\{1+\frac{1}{\eta \lambdamax }~,~n^3(\eta \lambdamax )^2\right\}
		~\ge~\frac{c\eta^2 G^2 \lambdamax }{16}\cdot \left(1+\frac{1}{\eta \lambdamax }\right)~=~
		\frac{c\eta G^2 }{16}\left(\eta \lambdamax +1\right).~
		\]
		Since $\eta\geq \frac{1}{\lambda nk}$, 
		this is at least $\frac{cG^2}{16\lambda nk}\left(\frac{\lambdamax }{\lambda nk}+1\right)$. Since we assume $\frac{\log(nk)L}{\lambda nk}\le1$ which entails $nk\ge \frac{\lambdamax}{\lambda}$, we can further lower bound it (without losing much) by $\frac{cG^2}{16\lambda nk}$. 
	\end{itemize}
	
	Combining the cases, we get that regardless of how we choose $\eta$, for some numerical constant $c''>0$, it holds that
	\[
	~\E[F(\bx_k)]~\geq~ c''\cdot \min\left\{\frac{G^2}{\lambda nk}~,~\frac{G^2 \lambdamax }{\lambda^2 nk^2}\right\}
	~=~ c''\cdot \frac{G^2}{\lambda nk}\cdot \min\left\{1~,~\frac{\lambdamax /\lambda}{k}\right\}~.
	\]
\end{proof}

\subsection{Proof of \thmref{thm:rr_lower_bound}}\label{subsec:rr_lower_proof}

\begin{proof}[\unskip\nopunct]
	As in the proof of \thmref{thm:ss_lower_bound}, we will assume w.l.o.g.\ that $n$ is even. Recall the function $F$ defined in \eqref{eq:rr_construction} by
	\[
	~F(\bx)~=~\frac{1}{n}\sum_{i=1}^{n}f_i(\bx)~=~\frac{\lambda}{2}x_1^2+\frac{\lambdamax}{2}x_2^2+\frac{\lambdamax}{4}x_3^2~,
	\]
	where for each $i$,
	\begin{equation*}
		~f_{i}(\bx)~=~f_i(x_1,x_2,x_3)~\coloneqq~\frac{\lambda}{2}x_1^2+\frac{\lambdamax}{2}x_2^2 + \begin{cases}\frac{G}{2}x_2 + \frac{\lambdamax}{2}x_3^2 + \frac{G}{2}x_3& i\leq \frac{n}{2}\\-\frac{G}{2}x_2 - \frac{G}{2}x_3& i>\frac{n}{2}\end{cases}~.
	\end{equation*}
	Consider the initialization point
	\[
		~\bx_0=(x_{0,1}, x_{0,2}, x_{0,3}) = \p{\frac{G}{\lambda}~,~0~,~0}~.
	\]
	Note that the above functions satisfy $\norm{\nabla f_i(\bx^*)} = G/\sqrt{2} \le G$ for all $i\in[n]$, and that $\norm{\nabla F(\bx_0)} = G$, therefore Assumption \ref{as:lower_bounds} is satisfied. Our proof will analyze the convergence of random reshuffling SGD under the assumption that $\eta$ belongs to some interval in a partition of the positive real line. For each such interval in the partition, we will take the worst lower bound (i.e.\ the largest lower bound) along each dimension, where our final lower bound will be the minimum among the bounds derived on each interval.
	
	We begin with deriving an expression for $\bx_k$, the iterate after performing $k$ epochs. In a single epoch, after $n$ iterations, the relationship between the first and last iterates in the epoch satisfy
	\begin{align}
	~x_{t+1,1} ~&=~ (1-\eta\lambda)^n \cdot x_{t,1}~,\nonumber\\
	x_{t+1,2} ~&=~ (1-\eta \lambdamax)^n\cdot x_{t,2}+\frac{\eta 
		G}{2}\sum_{i=0}^{n-1}(1-2\sigma_i)(1-\eta \lambdamax)^{n-i-1}~,\nonumber\\
	x_{t+1,3} ~&=~ \prod_{i=0}^{n-1}(1-\eta\lambdamax\sigma_i)\cdot x_{t,3} + \frac{\eta G}{2}\sum_{i=0}^{n-1}(1-2\sigma_i)\prod_{j=i+1}^{n-1}(1-\eta\lambdamax\sigma_j)~,\label{eq:ignored_squared_term}
	\end{align}
	where $\sigma_0,\ldots,\sigma_{n-1}$ are a random permutation of $\frac{n}{2}$ $1$'s and $\frac{n}{2}$ $0$'s.
	Squaring and taking expectation on both sides, using the fact that $\E[1-2\sigma_i]=0$ and that $\bx_t$ is independent of the permutation sampled at epoch $t+1$, we have
	\begin{align}
	~\E\pcc{x_{t+1,1}^2} ~&=~ (1-\eta\lambda)^{2n}x_{t,1}^2 ~,\nonumber\\ ~\E\pcc{x_{t+1,2}^2} ~&=~ (1-\eta \lambdamax)^{2n} x_{t,2}^2 +\frac{\eta^2G^2}{4}\beta_{n,\eta,\lambdamax}~,\nonumber\\
	~\E\pcc{x_{t+1,3}^2} ~&\ge~ \E\pcc{\prod_{i=0}^{n-1}(1-\eta\lambdamax\sigma_i)^2}\E\pcc{x_{t,3}^2} + \eta G\E\pcc{\sum_{i=0}^{n-1}(1-2\sigma_i)\prod_{j=i+1}^{n-1}(1-\eta\lambdamax\sigma_j)}\E\pcc{x_{t,3}}~\label{eq:third_dim},
	\end{align}
	where $\beta_{n,\eta,\lambdamax}$ is defined in \eqref{eq:beta}. Unfolding the recursions above for the first two dimensions, we get that after $k$ epochs
	\begin{align*}
	\E\pcc{x_{k,1}^2} ~&=~ \p{1-\eta\lambda}^{2nk}x_{0,1}^2 ~,\\ \E\pcc{x_{k,2}^2} ~&=~ (1-\eta \lambdamax)^{2nk} x_{0,2}^2 +\frac{\eta^2 
		G^2}{4}\cdot\frac{1-(1-\eta \lambdamax )^{2nk}}{1-(1-\eta \lambdamax )^{2n}}\beta_{n,\eta,\lambdamax}~.
	\end{align*}
	Recalling that $F(\bx)=\frac{\lambda}{2}x_1^2+\frac{\lambdamax}{2}x_2^2 + \frac{\lambdamax}{4}x_3^2\ge \frac{\lambda}{2}x_1^2+\frac{\lambdamax}{2}x_2^2$ and combining with the above, we get
	\begin{equation}\label{eq:two_dims}
	\E\pcc{F(\bx_k)} ~\ge~ \frac{\lambda}{2}\p{1-\eta\lambda}^{2nk}x_{0,1}^2 ~+~ \frac{\lambdamax}{2}(1-\eta \lambdamax)^{2nk} x_{0,2}^2 ~+~ \frac{\lambdamax\eta^2 
		G^2}{8}\cdot\frac{1-(1-\eta \lambdamax )^{2nk}}{1-(1-\eta \lambdamax )^{2n}}\beta_{n,\eta,\lambdamax}~.
	\end{equation}
	We now move to a case analysis based on the value of $\eta$:
	\begin{itemize}[leftmargin=*]
		\item
		If $ \eta\le\frac{1}{\lambda nk} $, we focus on the first dimension. Recall that $x_{0,1}=\frac{G}{\lambda}$ and compute
		\begin{align*}
			\E[F(\bx_k)] ~&\ge~ \frac{\lambda x_{0,1}^2}{2}(1-\eta\lambda)^{2nk} 	~\ge~ 
			\frac{\lambda x_{0,1}^2}{2}\left(1-\frac{1}{nk}\right)^{2nk}\\ &\ge~
			\frac{\lambda x_{0,1}^2}{2}\left(\frac{1}{4}\right)^2~=~\frac{\lambda x_{0,1}^2}{32}~=~\frac{G^2}{32\lambda} ~\ge~ \frac{G^2}{32\lambda nk}~.
		\end{align*}
		\item 
		If $ \eta> \frac{1}{\lambda nk}$ as well as $\eta < \frac{1}{\lambdamax n}$ (assuming this range exists, namely when $k>\lambdamax/\lambda$), we focus on the second and third dimensions, each resulting in a different dependence on $n,k$. Starting with the second dimension, we have by Bernoulli's inequality that $(1-\eta\lambdamax)^{2n}\ge1-2n\eta \lambdamax>0$, as well as $(1-\eta \lambdamax)^{2nk}\leq (1-\eta \lambda)^{2nk}\leq (1-1/nk)^{2nk}\leq \exp(-2)$, implying that
		\[
		\E[F(\bx_k)]~\ge~ \frac{\eta^2 G^2 \lambdamax}{8}  \cdot \frac{1-\exp(-2)}{1-(1-2n\eta\lambdamax)} \beta_{n,\eta,\lambdamax} ~=~ \frac{\eta^2G^2 \lambdamax(1-\exp(-2))}{16n\eta \lambdamax}\cdot\beta_{n,\eta,\lambdamax}~.
		\]
		Plugging in \eqref{eq:betalb}, and noting that $ \eta < \frac{1}{\lambdamax n}$, it is easily verified that $\beta_{n,\eta,\lambdamax}\ge c\cdot\min\{1+1/\eta \lambdamax,n^3(\eta \lambdamax)^2\}= cn^3\eta^2 \lambdamax^2 $. 
		This implies that the displayed equation above is at least
		\[
		~c^{\prime}\frac{\eta^2 G^2 \lambdamax}{n\eta \lambdamax}\cdot n^3 \eta ^2 \lambdamax^2 ~=~ c^{\prime}\eta^3n^2G^2 \lambdamax^2~,
		\]
		for some constant $ c^{\prime} $. Since $ \eta> \frac{1}{\lambda nk}$, this is at least
		\[
		~c^{\prime}\frac{n^2G^2 \lambdamax^2}{\lambda^3n^3k^3} = c^{\prime}\frac{G^2 \lambdamax^2}{\lambda^3nk^3}~.
		\]
		
		Moving to the third dimension, we assume w.l.o.g.\ that $k\ge n$ (since otherwise the previous bound will be larger than the one to follow). Using Propositions \ref{prop:stochastic_terms} and \ref{prop:xt_bound} and \lemref{lem:deterministic_prod} on \eqref{eq:third_dim} (recall that $x_{0,3}=0$ by our assumption), we have
		\begin{align*}
		~\E\pcc{x_{t+1,3}^2} ~&\ge~ \p{1-\frac{\eta\lambdamax n}{2}}^2\E\pcc{x_{t,3}^2} + \frac{\eta^3G^2\lambdamax n}{128}\p{1-\p{1-\frac{\eta\lambdamax n}{2}}^t}~\\
		&\ge~ \p{1-\eta\lambdamax n}\E\pcc{x_{t,3}^2} + \frac{\eta^3G^2\lambdamax n}{128}\p{1-\p{1-\frac{\eta\lambdamax n}{2}}^t}~.
		\end{align*}
		Unfolding the recursion above yields
		\begin{align}
		~\E\pcc{x_{k,3}^2} ~&\ge~ \p{1-\eta\lambdamax n}^kx_{0,3}^2 + \frac{\eta^3G^2\lambdamax n}{128} \sum_{t=0}^{k-1}\p{1-\eta\lambdamax n}^{k-t-1} \p{1-\p{1-\frac{\eta\lambdamax n}{2}}^t}  ~\nonumber\\
		&\ge~ \frac{\eta^3G^2\lambdamax n}{128} \sum_{t=\floor{k/2}}^{k-1}\p{1-\eta\lambdamax n}^{k-t-1} \p{1-\p{1-\frac{\eta\lambdamax n}{2}}^t}  ~\nonumber\\
		&\ge~ \frac{\eta^3G^2\lambdamax n}{128} \sum_{t=\floor{k/2}}^{k-1}\p{1-\eta\lambdamax n}^{k-t-1} \p{1-\p{1-\frac{\eta\lambdamax n}{2}}^{\floor{k/2}}}~,\label{eq:opposite_geo_sums}
		\end{align}
		where we note that the above sum is not empty since we assume $k\ge n\ge2$. We now have
		\begin{align*}
		~\p{1-\frac{\eta\lambdamax n}{2}}^{\floor{k/2}} ~&=~ \p{1-\frac{\eta\lambdamax n}{2\cdot\floor{k/2}}\cdot\floor{k/2}}^{\floor{k/2}} ~\le~\exp\p{-\frac{\eta\lambdamax n \floor{k/2}}{2}}\\ 
		&\le~ \exp\p{-\frac{\lambdamax\floor{k/2}}{2\lambda k}} ~\le~ \exp\p{-\frac{\lambdamax}{8\lambda}} ~\le~ \exp\p{-\frac18} ~\le~ 0.9~,
		\end{align*}
		where the first inequality is due to $(1-x/y)^y\le\exp(-x)$ for all $x,y>0$ and the second inequality is by the assumption $\eta\ge\frac{1}{\lambda nk}$. The above entails
		\begin{equation}
		~1-\p{1-\frac{\eta\lambdamax n}{2}}^{\floor{k/2}} ~\ge~ 0.1~,\label{eq:tenth_lower_bound}
		\end{equation}
		which by plugging into \eqref{eq:opposite_geo_sums} yields
		\begin{align*}
		~\E\pcc{x_{k,3}^2} ~&\ge~ \frac{\eta^3G^2\lambdamax n}{1280} \sum_{t=\floor{k/2}}^{k-1}\p{1-\eta\lambdamax n}^{k-t-1}\\
		&=~ \frac{\eta^3G^2\lambdamax n}{1280} \cdot \frac{1-(1-\eta\lambdamax n)^{k-\floor{k/2}}}{\eta\lambdamax n}\\
		&\ge~ \frac{\eta^2G^2}{1280} \cdot \p{1-\p{1-\frac{\eta\lambdamax n}{2}}^{\floor{k/2}}} ~\ge~ \frac{\eta^2G^2}{12800}~,
		\end{align*}
		where the last inequality is a second application of \eqref{eq:tenth_lower_bound}. We now conclude with the assumption $\eta\ge\frac{1}{\lambda nk}$ to get
		\[
		~\E\pcc{F(\bx)} ~\ge~ \frac{\lambdamax}{4}\E\pcc{x_{k,3}^2} ~\ge~ c\frac{G^2\lambdamax}{\lambda^2n^2k^2} ~,
		\]
		where $c=\frac{1}{51200}$.
		\item 
		If $ \eta> \frac{1}{\lambda nk}$ as well as $\eta \geq \frac{1}{\lambdamax n}$, we focus on the remainder term of the second dimension in \eqref{eq:two_dims}.
		Noting that $\frac{1-(1-\eta \lambdamax)^{2nk}}{1-(1-\eta \lambdamax)^{2n}}=\sum_{i=0}^{k-1}\left((1-\eta \lambdamax)^{2n}\right)^i\geq (1-\eta \lambdamax)^0~=~1$, we have
		\[
		\E[F(\bx_k)]~\geq~ \frac{\eta^2 G^2 \lambdamax}{8}\beta_{n,\eta,\lambdamax}~.
		\]
		By the assumption that $\eta\geq \frac{1}{\lambdamax n}$, we have that $ n^3(\eta \lambdamax)^2\ge 1/\eta \lambdamax$ as well as $n^3(\eta \lambdamax)^2\geq 1$. Using this and \eqref{eq:betalb}, the above is at least
		\[
		\frac{c \eta^2 G^2 \lambdamax }{8}\min\left\{1+\frac{1}{\eta \lambdamax }~,~n^3(\eta \lambdamax )^2\right\}
		~\ge~\frac{c\eta^2 G^2 \lambdamax }{16}\cdot \left(1+\frac{1}{\eta \lambdamax }\right)~=~
		\frac{c\eta G^2 }{16}\left(\eta \lambdamax +1\right).~
		\]
		Since $\eta\geq \frac{1}{\lambda nk}$, 
		this is at least $\frac{cG^2}{16\lambda nk}\left(\frac{\lambdamax }{\lambda nk}+1\right)$. Since we assume $\frac{\log(nk)L}{\lambda nk}\le1$ which entails $nk\ge \frac{\lambdamax}{\lambda}$, we can further lower bound it (without losing much) by $\frac{cG^2}{16\lambda nk}$. 
	\end{itemize}
\end{proof}

\subsection{Proof of \thmref{thm:ss_upper_bound_commuting}}\label{subsec:ss_upper_proof}

Before we prove the theorem, we will first state the following result which handles the one-dimensional case.
\begin{theorem}\label{thm:ss_upper_bound}
	Suppose $F(x) \coloneqq \frac{\bar{a}}{2}x^2$ and $f_i(x)=\frac{a_i}{2}x^2-b_ix$, where $\bar{a}=\frac1n\sum_{i=1}^{n}a_i$ satisfy Assumption \ref{as:upper_bounds}, and fix the step size $\eta=\frac{\log(nk)}{\lambda nk}$. Then for any $\delta\in (0,1)$, with probability at least $1-\delta$ over the choice of the permutation $\sigma$, single shuffling SGD satisfies
	\[
	~F(x_k)~\leq~c\cdot \log^2\left(\frac{8n}{\delta}\right)\cdot \log^2(nk)\cdot \frac{G^2}{\lambda nk}\cdot \min\left\{1~,~\frac{\bar{a}/\lambda}{k}\right\}~,
	\]
	where $c>0$ is a universal constant. Moreover, this also entails
	\[
	~\E\pcc{F(x_k)}~\leq~\tilde{\Ocal}\p{\frac{G^2}{\lambda n k}\cdot\min\set{1 ~,~ \frac{\bar{a}/\lambda}{k}}}~,
	\]
	where the $\tilde{\Ocal}$ hides a universal constant and factors logarithmic in $n,k,\bar{a},\lambda$ and their inverses.
\end{theorem}

\begin{proof}[\unskip\nopunct]
	Before we prove the above theorem, we shall first explain why it implies \thmref{thm:ss_upper_bound_commuting}. Since the matrices $A_1,\ldots,A_n$ commute, then they are simultaneously diagonalizable (e.g., \citet[Thm.~1.3.21]{horn2013matrix}). Thus, there exists a matrix $P$ such that $P^{-1}A_iP$ is diagonal for all $i\in[n]$. Moreover, since $A_i$ are all symmetric, we may choose such $P$ which is also orthogonal. By \appref{app:conj}, we may transform our problem to another quadratic formulation having the same sub-optimality rate and where Assumption \ref{as:upper_bounds} is preserved. Following the above reasoning, we may assume w.l.o.g.\ that $A_i$ is diagonal for all $i\in[n]$.
	
	For some $j\in[d]$ and $t\in[k]$, let $a_j$ and $x_{j,t}$ denote the $j$-th diagonal value of $A$ and $j$-th coordinate of $\bx_t$ (the iterate after the $t$-th epoch), respectively. We now explain why we may assume that $\bb=\mathbf{0}$. As assumed in \citet{safran2020good}, mapping $f_i(\bx)\mapsto \tilde{f}_i(\bx-A^{-1}\bb)$ for all $i\in[n]$ simply translates our problem so that $\bx^*=\mathbf{0}$. By mapping $\bx_0$ accordingly, we have that Assumption \ref{as:upper_bounds} is preserved, thus we may assume $\bb=\mathbf{0}$ w.l.o.g.\ which entails $\norm{A_i\bx^*-\bb_i}=\norm{\bb_i}\le G$ for all $i\in[n]$.
	
	Since we have now reduced our optimization problem to the form $\tilde{f}_i(\bx)=\frac12\bx^{\top}A_i\bx - \bb_i^{\top}\bx$ for diagonal $A_i$, we have that the partial derivatives w.r.t.\ each coordinate are independent of one another, thus we may apply \thmref{thm:ss_upper_bound} to each coordinate separately.
	Letting $F_j(x)=\frac12a_{j}x^2$, we compute
	\[
		~\E\pcc{F(\bx_k)} ~=~ \E\pcc{\frac12\bx_k^{\top}A\bx_k} ~=~ \sum_{j=1}^{d}\E\pcc{F_j(x_{j,k})}~.
	\]
	Recall that $\bx_0=(x_{1,0},\ldots,x_{d,0})$. The condition $\norm{\nabla F(\bx_0)}\le G$ implies that $\sum_{j=1}^{d}a_{j}^2x_{j,0}^2\le G^2$, and since $\lambda\le a_{j}\le\lambdamax$ for all $j\in[d]$, we get $\norm{\bx_0}\le\frac{G}{\lambda}$, which in particular implies $x_{j,0}\le \frac{G}{\lambda}$ for all $j\in[d]$. We now use \thmref{thm:ss_upper_bound} applied to each dimension separately and conclude that
	\[
		~\E\pcc{F(\bx_k)} ~=~ \sum_{j=1}^d \tilde{\Ocal}\p{\frac{G^2}{\lambda n k}\cdot\min\set{1 ~,~ \frac{a_{j}/\lambda}{k}}} ~\le~ \tilde{\Ocal}\p{\frac{G^2}{\lambda n k}\cdot\min\set{1 ~,~ \frac{\lambdamax/\lambda}{k}}}~,
	\]
	whereby the $\tilde{\Ocal}$ notation hides a linear term in $d$ which absorbs the sum over the coordinates.
\end{proof}

\begin{proof}[Proof of \thmref{thm:ss_upper_bound}]
	The beginning of the proof is based on deriving a closed-form expression for the iterate at the $k$-th epoch, $x_k$. To this end, we shall use the same derivation as in \citet{safran2020good}, given here for completeness. First, for a selected permutation $ \sigma_i:[n]\to[n] $ we have that the gradient update at iteration $ j $ in epoch $ i $ is given by
	\[
		~x_{new} ~=~ \p{1-\eta a_{\sigma_i(j)}}x_{old} + \eta b_{\sigma_i(j)}~.
	\]
	Repeatedly applying the above relation, we have that at the end of each epoch the relation between the iterates $ x_t $ and $ x_{t+1} $ is given by
	\[
		~x_{t+1} ~=~ \prod_{j=1}^{n}\p{1-\eta a_{\sigma_{t+1}(j)}}x_t + \eta\sum_{j=1}^{n}b_{\sigma_{t+1}(j)}\prod_{i=j+1}^{n}\p{1-\eta a_{\sigma_{t+1}(i)}}~.
	\]
	Letting
	\[
		~S ~\coloneqq~ \prod_{j=1}^{n}\p{1-\eta a_{\sigma_i(j)}} = \prod_{j=1}^{n}\p{1-\eta a_j}~
	\]
	and 
	\[
		~X_{\sigma_t}~\coloneqq~\sum_{j=1}^{n}b_{\sigma_{t}(j)}\prod_{i=j+1}^{n}\p{1-\eta a_{\sigma_{t}(i)}}~,
	\]
	this can be rewritten equivalently as
	\begin{equation}\label{eq:consecutive_iterates}
		~x_{t+1} = Sx_t +\eta X_{\sigma_{t+1}}~.
	\end{equation}
	Iteratively applying the above, we have after $ k $ epochs that
	\begin{equation}\label{eq:t_iterate}
	~x_k~=~ S^k x_0+\eta\sum_{i=1}^{k}S^{i-1}X_{\sigma}~=~ S^k x_0+\eta\cdot \frac{1-S^{k}}{1-S}X_{\sigma}~.
	\end{equation}
	Having derived a closed-form expression for $x_k$, we now turn to make a more careful analysis of the upper bound, improving upon the result of \citet{safran2020good}. Note that by our assumptions, $1\geq 1-\eta a_j\geq 1-\eta L\geq 0$ for all $j$, hence $S\in [0,1]$.	As a result, using the fact that $(r+s)^2\leq2(r^2+s^2)$, we have
	\begin{align}
	~F(x_k)~&=~\frac{\bar{a}}{2}x_k^2~\leq~\bar{a}\left(S^{2k}x_0^2+\eta^2 \left(\frac{1-S^k}{1-S}\right)^2 X_{\sigma}^2\right)~\leq~
	\bar{a}\left(S^{2k}\cdot \frac{G^2}{\lambda^2}+\eta^2 \left(\frac{1-S^k}{1-S}\right)^2 X_{\sigma}^2\right)\notag\\
	&\leq~
	S^{2k}\cdot \frac{\bar{a}G^2}{\lambda^2}+\frac{\bar{a}\eta^2}{(1-S)^2}\cdot X_{\sigma}^2
	~,\label{eq:objective}
	\end{align}
	whereby $x_0^2\le \frac{G^2}{\lambda^2}$ is due to Assumption \ref{as:upper_bounds}, which entails $\abs{\bar{a}x_0}\le G$ and thus $\abs{x_0}\le\frac{G}{\bar{a}}\le\frac{G}{\lambda}$. We now have
	\begin{align}
	~S^{2k}~&=~\prod_{j=1}^{n}(1-\eta a_j)^{2k}~\leq~ (1-\eta\bar{a})^{2nk}~=~
	\left(1-\frac{\bar{a}\log(nk)}{\lambda nk}\right)^{2nk}\nonumber\\
	&\leq~ \exp\p{\frac{-2\bar{a}\log(nk)}{\lambda}}~=~ \frac{1}{(nk)^{2\bar{a}/\lambda}} ~\le~ \frac{\lambda}{\bar{a}(nk)^2} ~\label{eq:S2k_bound},
	\end{align}
	where the first inequality is by the AM-GM inequality applied to $1-\eta a_1,\ldots,1-\eta a_n>0$, and the last inequality is due to $(nk)^2\ge4$ and the fact that $x^y\le x/y$ for all $x\in[0,0.25]$ and $y\ge1$.\footnote{To see this, we first have that the inequality is trivial when $y=1$. Assuming $y>1$, $x/y-x^y$ intersects the $x$ axis iff $x=0$ or $x=y^{1/(1-y)}$ which is at least $\exp(-1)\ge0.25$ for $y>1$, and thus we can verify that $x^y\le x/y$ for all $x\in[0,\exp(-1)]$ by establishing that $x/y-x^y$ is concave on $(0,\exp(-1))$.} Moreover, 
	\begin{equation}\label{eq:S_bound}
	~S~=~ \prod_{j=1}^{n}(1-\eta a_j)~=~\exp\left(\sum_{j=1}^{n}\log(1-\eta a_j)\right)~\leq~ \exp\left(-\eta\sum_{j=1}^{n}a_j\right)~=~\exp(-\eta \bar{a} n)~.
	\end{equation}
	Plugging the two displayed equations above into \eqref{eq:objective}, we get that
	\begin{equation}\label{eq:objective2}
	~F(x_k)~\leq~ \frac{G^2}{\lambda (nk)^2}+\frac{\bar{a}\eta^2}{(1-\exp(-\eta\bar{a}n))^2}\cdot X_{\sigma}^2~.
	\end{equation}
	To continue, we will use the following key technical lemma, which we shall use to upper bound $X_{\sigma}^2$ with high probability (using $\alpha_i:=\eta a_i$ and $\beta_i=b_i/G$ for all $i$):
	\begin{lemma}\label{lem:keyup}
		Let $\alpha_1,\beta_1,\ldots,\alpha_n,\beta_n$ be scalars such that for all $i$, $\alpha_i \in [0,1]$, $|\beta_i|\leq 1$ and $\sum_{i=1}^{n}\beta_i=0$. Then for any $\delta\in (0,1)$, with probability at least $1-\delta$, we have
		\[
		\left(\sum_{j=1}^{n}\beta_{\sigma(j)}\prod_{i=j+1}^{n}(1-\alpha_{\sigma(i)})\right)^2~\leq~ c\cdot\log^2\left(\frac{8n}{\delta}\right)\cdot \min\left\{\frac{1}{\bar{\alpha}}~,~n^3\bar{\alpha}^2\right\}
		\]
		where $\bar{\alpha}=\frac{1}{n}\sum_{i=1}^{n}\alpha_i$ and $c>0$ is a universal constant. 
	\end{lemma}
	The proof appears in \subsecref{subsec:lemkeyupproof}. We note that we did not try to optimize the log factor.
	
	\begin{remark}
		This upper bound complements Lemma 1 from \citet{safran2020good}, which analyzed the same key quantity in the special case where $\beta_i\in \{-1,+1\}$ and $\alpha_i=\bar{\alpha}$ are the same for all $i$, and showed (when $\bar{\alpha}\in [0,1]$) a \emph{lower bound} of $c'\cdot \min\left\{\frac{1}{\bar{\alpha}}~,~n^3\bar{\alpha}^2\right\}$ for some universal constant $c'>0$.\footnote{In \citet{safran2020good}, the exact lower bound is $c'\min\set{1+\frac{1}{\bar{\alpha}}~,~ n^3\bar{\alpha}^2}$, which is equivalent to $c''\min\left\{\frac{1}{\bar{\alpha}}~,~n^3\bar{\alpha}^2\right\}$ for some constant $c''>0$} This implies that our upper bound is tight up to constants and logarithmic factors.
	\end{remark}
	
	We now consider two cases, depending on the value of $\eta\bar{a}n$: 
	
	\begin{itemize}[leftmargin=*]
		\item \textbf{Case 1: $\eta\bar{a}n\leq \frac{1}{2}$.} By \lemref{lem:keyup}, with probability at least $1-\delta$,
		\[
		X_{\sigma}^2~\leq~ c\cdot \log^2(8n/\delta)\cdot G^2 n^3 (\eta \bar{a})^2~.
		\]
		In addition,
		\[
		\exp(-\eta \bar{a} n)~\leq~1-\frac{1}{2}\eta \bar{a} n~,
		\]
		due to the assumption $\eta\bar{a}n\leq \frac{1}{2}$ and the fact that $\exp(-z)\leq 1-\frac{1}{2}z$ for all $z\in [0,1/2]$. Plugging the two displayed equations above back into \eqref{eq:objective2}, we get that with probability at least $1-\delta$,
		\begin{align*}
		F(x_k)~&\leq~ \frac{G^2}{\lambda nk^2}+\frac{\bar{a}\eta^2}{(\eta \bar{a} n/2)^2}\cdot c\cdot \log^2(8n/\delta)\cdot G^2 n^3(\eta \bar{a})^2\\
		&=~ \frac{G^2}{\lambda (nk)^2}+4c\cdot\log^2(8n/\delta)\cdot\bar{a} n G^2 \eta^2\\
		&=~ \frac{G^2}{\lambda (nk)^2}+4c\cdot\log^2(8n/\delta)\cdot\bar{a} n G^2\cdot \left(\frac{\log(nk)}{\lambda nk}\right)^2\\
		&=~ \frac{G^2}{\lambda (nk)^2}+4c\cdot\log^2(8n/\delta)\cdot\log^2(n k)\cdot \frac{\bar{a}G^2}{\lambda^2 n k^2}\\
		&\leq~ \left(1+4c\cdot\log^2(8n/\delta)\cdot\log^2(n k)\right)\cdot \frac{\bar{a}G^2}{\lambda^2 n k^2}~,
		\end{align*}
		where in the last step we used the fact that $\frac{1}{n}\leq 1\leq \frac{\bar{a}}{\lambda}$. Likewise, bounding $X_{\sigma}^2$ in expectation using \propref{prop:random_variable_bounds} yields a bound of
		\[
		~\E\pcc{F(x_k)} ~\le~ \tilde{\Ocal}\p{\frac{\bar{a}G^2}{\lambda^2nk^2}}~.
		\]
		\item \textbf{Case 2: $\eta \bar{a} n>\frac{1}{2}$.} By \lemref{lem:keyup}, with probability at least $1-\delta$,
		\[
		X_{\sigma}^2~\leq~ \frac{c\log^2(8n/\delta)}{\eta\bar{a}}~.
		\]
		In addition, $\eta\bar{a}n=\frac{\log(nk)\bar{a}}{\lambda k}>\frac{1}{2}$. Plugging these back into \eqref{eq:objective2}, we get that
		\begin{align*}
		F(x_k)~&\leq~ \frac{G^2}{\lambda (nk)^2}+\frac{\bar{a}\eta^2}{(1-\exp(-1/2))^2}\cdot \frac{cG^2\log^2(8n/\delta)}{\eta\bar{a}}\\
		&\leq~ \frac{G^2}{\lambda (nk)^2}+\frac{cG^2\log^2(8n/\delta)}{(1-\exp(-1/2))^2}\cdot \eta\\
		&=~ \frac{G^2}{\lambda (nk)^2}+\frac{c\log^2(8n/\delta)\cdot\log(nk)}{(1-\exp(-1/2))^2}\cdot \frac{G^2}{\lambda n k}\\
		&\leq~ \left(1+\frac{c\log^2(8n/\delta)\cdot\log(nk)}{(1-\exp(-1/2))^2}\right)\cdot \frac{G^2}{\lambda n k}~.
		\end{align*}
		Likewise, bounding $X_{\sigma}^2$ in expectation using \propref{prop:random_variable_bounds} yields a bound of
		\[
		~\E\pcc{F(x_k)} ~\le~ \tilde{\Ocal}\p{\frac{G^2}{\lambda nk}}~.
		\]
	\end{itemize}
	
	To combine the two cases, we note that the condition $\eta \bar{a}n\leq \frac{1}{2}$ is equivalent to $\frac{\bar{a}\log(nk)}{\lambda k}\leq \frac{1}{2}$. In that case, we have $\frac{\log(nk)\bar{a}G^2}{\lambda^2 n k^2}~\leq~\frac{1}{2}\cdot \frac{G^2}{\lambda n k}$, and thus
	\[
	\frac{\log(nk)\bar{a}G^2}{\lambda^2 n k^2}~\leq~ \min\left\{\frac{\log(nk)\bar{a}G^2}{\lambda^2 n k^2}~,~\frac{G^2}{2\lambda nk}\right\}~=~ \frac{G^2}{\lambda nk}\cdot \min\left\{\frac{\log(nk)\bar{a}}{\lambda k}~,~\frac{1}{2}\right\}~.
	\]
	In the opposite case where $\frac{\bar{a}\log(nk)}{\lambda k}> \frac{1}{2}$, it follows that $\frac{G^2}{\lambda n k}~<~ \frac{2\log(nk) \bar{a}G^2}{\lambda^2 n k^2}$, and therefore
	\[
	\frac{G^2}{\lambda n k}~\leq~\min\left\{\frac{2\log(nk) \bar{a}G^2}{\lambda^2 n k^2}~,~\frac{G^2}{\lambda nk}\right\}~=~ \frac{G^2}{\lambda n k}\cdot \min\left\{\frac{2\log(nk)\bar{a}}{\lambda k}~,~1\right\}~.
	\]
	Plugging these two inequalities into the bounds obtained in the two cases above and simplifying a bit, the result follows. 
\end{proof}

\subsection{Proof of \thmref{thm:rr_upper_bound_commuting}}\label{subsec:rr_upper_proof}

\begin{proof}[\unskip\nopunct]
	Similarly to the proof of \thmref{thm:ss_upper_bound_commuting}, we first assume w.l.o.g.\ that $A_i$ is diagonal for all $i\in[n]$ and that $\bb=\mathbf{0}$, implying a per-coordinate gradient bound of $G$ (see the argument following \thmref{thm:ss_upper_bound} for justification). Under the same reasoning, the proof then follows from the following theorem.
\end{proof}

\begin{theorem}\label{thm:rr_upper_bound}
	Suppose $F(x) \coloneqq \frac{\bar{a}}{2}x^2$ and $f_i(x)=\frac{a_i}{2}x^2-b_ix$, where $\bar{a}=\frac1n\sum_{i=1}^{n}a_i$ satisfy Assumption \ref{as:upper_bounds}, and fix the step size $\eta=\frac{\log(nk)}{\lambda nk}$. Then random reshuffling SGD satisfies
	\[
	~\E\pcc{F(x_k)}~\leq~\tilde{\Ocal}\p{\frac{G^2}{\lambda n k}\cdot\min\set{1 ~,~ \frac{\bar{a}/\lambda}{nk} + \frac{\bar{a}^2/\lambda^2}{k^2}}}~.
	\]
	where the $\tilde{\Ocal}$ hides a universal constant and factors logarithmic in $n,k,\bar{a},\lambda$ and their inverses.
\end{theorem}

\begin{proof}
	Our analysis picks off from \citet[Eq.~(22)]{safran2020good}. However, for the sake of completeness we shall include the derivation of Eq.~(22) as was done in the above reference.
	
	Continuing from \eqref{eq:consecutive_iterates}, we square and take expectation on both sides to obtain
	\[
	~\E\pcc{x_{t+1}^2} ~=~ \E\pcc{\p{Sx_t +\eta X_{\sigma_{t+1}}}^2} = S^2\E[x_t^2] + 2\eta S\E\pcc{x_tX_{\sigma_{t+1}}} + \eta^2\E\pcc{X_{\sigma_{t+1}}^2}~.
	\]
	Since in random reshuffling the random component at iteration $ t+1 $, $ X_{\sigma_{t+1}} $, is independent of the iterate at iteration $ t $, $ x_t $, and by plugging $t=k$ into \eqref{eq:t_iterate}, the above equals
	\begin{align*}
	~\E\pcc{x_{t+1}^2} ~&=~ S^2\E[x_t^2] + 2\eta S\E\pcc{x_t}\E\pcc{X_{\sigma_{t+1}}} + \eta^2\E\pcc{X_{\sigma_{t+1}}^2}~\\
	&=~ S^2\E[x_t^2] + 2\eta S\E\pcc{S^tx_0 + \eta\sum_{i=1}^{t}S^{t-i}X_{\sigma_i}}\E\pcc{X_{\sigma_{t+1}}} + \eta^2\E\pcc{X_{\sigma_{t+1}}^2} ~\\
	&=~ S^2\E[x_t^2] + 2\eta S^{t+1}x_0\E\pcc{X_{\sigma_{t+1}}} + 2\eta^2\sum_{i=1}^{t}S^{t-i+1}\E\pcc{X_{\sigma_i}}\E\pcc{X_{\sigma_{t+1}}} + \eta^2\E\pcc{X_{\sigma_{t+1}}^2} ~\\
	&=~ S^2\E[x_t^2] + 2\eta S^{t+1}x_0\E\pcc{X_{\sigma_{1}}} + 2\eta^2\sum_{i=1}^{t}S^{t-i+1}\E\pcc{X_{\sigma_1}}^2 + \eta^2\E\pcc{X_{\sigma_{1}}^2}~,
	\end{align*}
	where the last equality is due to $ X_{\sigma_i} $ being i.i.d.\ for all $ i $. Recursively applying the above relation and taking absolute value, we obtain
	\begin{equation*}
	~\E\pcc{x_{k}^2} ~=~ S^{2k}x_0^2 + 2\eta x_0\E\pcc{X_{\sigma_1}}\sum_{j=0}^{k-1}S^{k+j} + 2\eta^2\E\pcc{X_{\sigma_1}}^2\sum_{j=0}^{k-1}S^{2j}\sum_{i=1}^{k-j-1}S^{i} + \eta^2\E\pcc{X_{\sigma_{1}}^2}\sum_{j=0}^{k-1}S^{2j}~.
	\end{equation*}
	
	Having derived the bound appearing in \citet[Eq.~22]{safran2020good}, we now turn to improve their result by refining the upper bound as follows. We have that the above implies
	\begin{align*}
	~\E\pcc{x_{k}^2} ~&=~ S^{2k}x_0^2 + 2\eta x_0\E\pcc{X_{\sigma_1}}S^k\frac{1-S^k}{1-S} + 2\eta^2\E\pcc{X_{\sigma_1}}^2S\sum_{j=0}^{k-1}S^{2j}\cdot \frac{1-S^{k-j-1}}{1-S}  + \eta^2\E\pcc{X_{\sigma_{1}}^2}\frac{1-S^{2k}}{1-S}~\\
	&\le~ S^{2k}x_0^2 + 2\eta x_0\E\pcc{X_{\sigma_1}}S^k\frac{1}{1-S} + 2\eta^2\E\pcc{X_{\sigma_1}}^2\sum_{j=0}^{k-1}S^{2j}\cdot \frac{1}{1-S}  + \eta^2\E\pcc{X_{\sigma_{1}}^2}\frac{1}{1-S}~\\
	&\le~ S^{2k}x_0^2 + 2\eta x_0\E\pcc{X_{\sigma_1}}S^k\frac{1}{1-S} + 2\eta^2\E\pcc{X_{\sigma_1}}^2 \frac{1-S^{2k}}{(1-S)^2}  + \eta^2\E\pcc{X_{\sigma_{1}}^2}\frac{1}{1-S}~\\
	&\le~ S^{2k}x_0^2 + 2\eta x_0\E\pcc{X_{\sigma_1}}S^k\frac{1}{1-S} + 2\eta^2\E\pcc{X_{\sigma_1}}^2 \frac{1}{(1-S)^2}  + \eta^2\E\pcc{X_{\sigma_{1}}^2}\frac{1}{1-S}~.
	\end{align*}
	Using the fact that $(r+s)^2\leq2(r^2+s^2)$ for $r=S^kx_0$ and $s=\eta\E\pcc{X_{\sigma_{1}}^2}/(1-S)$, we have
	\[
	~\E\pcc{x_{k}^2} ~\le~ 2S^{2k}x_0^2 + 3\eta^2\E\pcc{X_{\sigma_1}}^2 \frac{1}{(1-S)^2}  + \eta^2\E\pcc{X_{\sigma_{1}}^2}\frac{1}{1-S}~.
	\]
	Next, we use the assumption $x_0^2\le \frac{G^2}{\lambda^2}$ (which follows from Assumption \ref{as:upper_bounds}, since it entails $\abs{\bar{a}x_0}\le G$ and thus $\abs{x_0}\le\frac{G}{\bar{a}}\le\frac{G}{\lambda}$), along with Equations (\ref{eq:S2k_bound}) and (\ref{eq:S_bound}) to upper bound the above by
	\begin{equation}\label{eq:pre_prop_upper_bound}
	~\frac{2G^2}{\bar{a}\lambda n^2k^2} + \frac{3\eta^2\E\pcc{X_{\sigma_1}}^2}{(1-\exp(-\eta\bar{a}n))^2}  + \frac{\eta^2\E\pcc{X_{\sigma_{1}}^2}}{1-\exp(-\eta\bar{a}n)}~.
	\end{equation}
	We now consider two cases, depending on the value of $\eta\bar{a}n$, using \propref{prop:random_variable_bounds} by letting $\alpha_j=\eta a_j$ and $\beta_j=b_j/G$:
	\begin{itemize}[leftmargin=*]
		\item 
		\textbf{Case 1: $\eta\bar{a}n \le \frac12$.} We have that \eqref{eq:pre_prop_upper_bound} is upper bounded by
		\begin{align*}
		~&\frac{2G^2}{\bar{a}\lambda n^2k^2} + \frac{12G^2\eta^4n^2\bar{a}^2}{(1-\exp(-\eta\bar{a}n))^2}  + c_2\log^2\left(\frac{8}{\eta n\bar{a}}\right)\cdot\frac{G^2\eta^4n^3\bar{a}^2}{1-\exp(-\eta\bar{a}n)}~\\
		\le~& \frac{2G^2}{\bar{a}\lambda n^2k^2} + 48G^2\eta^2  + 2c_2\log^2\left(\frac{8}{\eta n\bar{a}}\right)\cdot G^2\eta^3n^2\bar{a}~ \\
		\le~&\tilde{\Ocal}\p{\frac{G^2}{\bar{a}\lambda n^2k^2} + \frac{G^2}{\lambda^2n^2k^2} + \frac{G^2\bar{a}}{\lambda^3nk^3}} ~\le \tilde{\Ocal}\p{\frac{G^2}{\lambda^2n^2k^2} + \frac{G^2\bar{a}}{\lambda^3nk^3}}~,
		\end{align*}
		where we used the inequality $\exp(-x)\le1-x/2$ which holds for all $x\in[0,1/2]$, and the fact that $\frac{\bar{a}}{\lambda}\ge1$. Using the definition of $F$ we get
		\[
		\E\pcc{F(x_k)} ~=~ \frac{\bar{a}}{2}\E\pcc{x_k^2} ~\le~ \tilde{\Ocal}\p{\frac{G^2\bar{a}}{\lambda^2n^2k^2} + \frac{G^2\bar{a}^2}{\lambda^3nk^3}}~.
		\]
		\item 
		\textbf{Case 2: $\eta\bar{a}n>\frac12$.} In this case we have that \eqref{eq:pre_prop_upper_bound} is upper bounded by
		\begin{align*}
		~&\frac{2G^2}{\bar{a}\lambda n^2k^2} + c_1^2\log^2\p{\sqrt{2\bar{\alpha}}\cdot8n^{2}}\frac{3G^2\eta^2}{\eta\bar{a}(1-\exp(-1/2))^2}  + c_3^2\log^2\left(8n^2\bar{\alpha}^2\right)\cdot\frac{G^2\eta^2}{\eta\bar{a}(1-\exp(-1/2))}~\\
		\le~&\tilde{\Ocal}\p{\frac{G^2}{\bar{a}\lambda n^2k^2} + \frac{G^2\eta}{\bar{a}}} ~\le~ \tilde{\Ocal}\p{\frac{G^2}{\bar{a}\lambda n^2k^2} + \frac{G^2}{\bar{a}\lambda nk}} ~\le~ \tilde{\Ocal}\p{ \frac{G^2}{\bar{a}\lambda nk}}~.
		\end{align*}
		Using the definition of $F$ we get
		\[
		\E\pcc{F(x_k)} ~=~ \frac{\bar{a}}{2}\E\pcc{x_k^2} ~\le~ \tilde{\Ocal}\p{ \frac{G^2}{\lambda nk}}~.
		\]
	\end{itemize}
	
	To combine the two cases, we note that the condition $\eta \bar{a}n\leq \frac{1}{2}$ is equivalent to $\frac{\bar{a}\log(nk)}{\lambda k}\leq \frac{1}{2}$. In that case we have
	\[
	~\frac{\log(nk)\bar{a}G^2}{\lambda^2 n^2 k^2} + \frac{\log^2(nk)\bar{a}^2G^2}{\lambda^3 n k^3} ~\le~
	\frac{G^2}{2\lambda nk} + \frac{G^2}{4\lambda n k} ~\le~
	\frac{G^2}{\lambda n k}~,
	\]
	implying
	\begin{align*}
	~\frac{\log(nk)\bar{a}G^2}{\lambda^2 n^2 k^2} + \frac{\log^2(nk)\bar{a}^2G^2}{\lambda^3 n k^3} ~&\le~ \min\set{\frac{\log(nk)\bar{a}G^2}{\lambda^2 n^2 k^2} + \frac{\log^2(nk)\bar{a}^2G^2}{\lambda^3 n k^3} ~,~ \frac{G^2}{\lambda n k}} ~\\
	&=~
	\frac{G^2}{\lambda n k}\cdot\min\set{1 ~,~ \frac{\log(nk)\bar{a}}{\lambda n k} + \frac{\log^2(nk)\bar{a}^2}{\lambda^2k^2}}~.
	\end{align*}
	In the opposite case where $\frac{\bar{a}\log(nk)}{\lambda k}> \frac{1}{2}$, it follows that $1<\frac{2\bar{a}\log(nk)}{\lambda k}$, and therefore
	\begin{align*}
	~\frac{G^2}{\lambda n k} ~&\le~ \frac{4\bar{a}^2\log^2(nk)G^2}{\lambda^3 n k^3} ~\le~ \frac{4\bar{a}\log(nk)G^2}{\lambda^2 n^2 k^2} + \frac{4\bar{a}^2\log^2(nk)G^2}{\lambda^3 n k^3} ~\\
	&=~ 4\frac{G^2}{\lambda n k}\cdot\min\set{1 ~,~ \frac{\log(nk)\bar{a}}{\lambda n k} + \frac{\log^2(nk)\bar{a}^2}{\lambda^2k^2}} ~.
	\end{align*}
	Plugging these two inequalities into the bounds obtained in the two cases above and absorbing logarithmic terms into the big $\tilde{\Ocal}$ notation, the result follows. 
\end{proof}

\section{Technical Lemmas}

\subsection{Proofs of Propositions}
\begin{proposition}\label{prop:stochastic_terms}
	Suppose $\sigma_0,\ldots,\sigma_{n-1}$ is a random permutation of $\frac n2$ $0$'s and $\frac n2$ $1$'s and $\eta\le\frac{1}{\lambdamax n}$. Then
	\[
	~\E\pcc{\p{\prod_{i=0}^{n-1}(1-\eta\lambdamax\sigma_i)} \p{\sum_{i=0}^{n-1}(1-2\sigma_i) \prod_{j=i+1}^{n-1} (1-\eta\lambdamax\sigma_j)}} ~\le~ -\frac{1}{16}\eta\lambdamax n~. 
	\]
\end{proposition}

\begin{proof}
	Starting with the first multiplicand, we have from \lemref{lem:deterministic_prod} that it is lower bounded by $0.5$ deterministically, as it does not depend on the permutation sampled, thus we can take it outside the expectation. At this point, the statement in the proposition reduces to showing that
	\[
	~\E\pcc{\sum_{i=0}^{n-1}(1-2\sigma_i) \prod_{j=i+1}^{n-1} (1-\eta\lambdamax\sigma_j)} ~\le~ -\frac{1}{8}\eta\lambdamax n~, 
	\]
	which follows immediately from \lemref{lem:sum_prod_expect}.
\end{proof}

\begin{proposition}\label{prop:xt_bound}
	Suppose $\sigma_0,\ldots,\sigma_{n-1}$ is a random permutation of $\frac n2$ $0$'s and $\frac n2$ $1$'s and $\eta\le\frac{1}{\lambdamax n}$. Let
	\[
	~x_{t+1} ~=~ \prod_{i=0}^{n-1}(1-\eta\lambdamax\sigma_i)\cdot x_t + \frac{\eta G}{2}\sum_{i=0}^{n-1}(1-2\sigma_i)\prod_{j=i+1}^{n-1}(1-\eta\lambdamax\sigma_j)~,
	\]
	where $x_0=0$. Then 
	\[
		~\E\pcc{x_k} ~\le~ -\frac{\eta G}{8}\p{1-\p{1-\frac{\eta\lambdamax n}{2}}^k} ~. 
	\]
\end{proposition}

\begin{proof}
	Taking expectation on both sides, using the fact that the iterate at the $t$-th epoch, $x_t$ is independent of the permutation sampled at epoch $t+1$, we have
	\[
	~\E\pcc{x_{t+1}} ~=~ \E\pcc{x_t}\E\pcc{\prod_{i=0}^{n-1}(1-\eta\lambdamax\sigma_i)} + \frac{\eta G}{2}\E\pcc{\sum_{i=0}^{n-1}(1-2\sigma_i)\prod_{j=i+1}^{n-1}(1-\eta\lambdamax\sigma_j)}~.
	\]
	Recall that $x_0=0$. Using Lemmas \ref{lem:deterministic_prod} and \ref{lem:sum_prod_expect}, we have by a simple inductive argument that $\E\pcc{x_t} \le 0$ for all $t$, and therefore
	\[
	~\E\pcc{x_{t+1}} ~\le~ \p{1-\frac{\eta\lambdamax n}{2}}\E\pcc{x_t} - \frac{\eta^2\lambdamax nG}{16}~.
	\]
	Unfolding the above recursion, we have
	\begin{align*}
		\E\pcc{x_k} ~&\le~ \p{1-\frac{\eta\lambdamax n}{2}}^{k}x_0 - \frac{\eta^2\lambdamax nG}{16}\sum_{i=0}^{k-1}\p{1-\frac{\eta\lambdamax n}{2}}^i\\
		&=~ -\frac{\eta^2\lambdamax nG}{16}\sum_{i=0}^{k-1}\p{1-\frac{\eta\lambdamax n}{2}}^i
		~=~ -\frac{\eta^2\lambdamax nG}{16} \cdot\frac{1-\p{1-\frac{\eta\lambdamax n}{2}}^k}{0.5\eta\lambdamax n}\\
		&= ~ -\frac{\eta G}{8}\p{1-\p{1-\frac{\eta\lambdamax n}{2}}^k}~.
	\end{align*}
\end{proof}

\begin{proposition}\label{prop:random_variable_bounds}
	Let $\alpha_1,\beta_1,\ldots,\alpha_n,\beta_n$ be scalars such that for all $i$, $\alpha_i \in [0,1]$, $|\beta_i|\leq 1$ and $\sum_{i=1}^{n}\beta_i=0$. Then
	\begin{itemize}
		\item 
		\[
		~\E\pcc{\abs{\sum_{j=1}^{n}\beta_{\sigma(j)}\prod_{i=j+1}^{n}(1-\alpha_{\sigma(i)})}}~\leq~
		\begin{cases}
		~2n\bar{\alpha} ~&~ n\bar{\alpha} ~\le~ \frac12~\\
		~c_1 \frac{\log\p{\sqrt{2\bar{\alpha}}\cdot8n^{2}}}{\sqrt{\bar{\alpha}}} ~&~ n\bar{\alpha} ~>~ \frac12~
		\end{cases}~,
		\]
		\item 
		\[
		~\E\pcc{\left(\sum_{j=1}^{n}\beta_{\sigma(j)}\prod_{i=j+1}^{n}(1-\alpha_{\sigma(i)})\right)^2}~\leq~
		\begin{cases}
		~c_2\cdot\log^2\left(\frac{8}{n\bar{\alpha}}\right)\cdot n^3\bar{\alpha}^2 ~&~ n\bar{\alpha} ~\le~ \frac12~\\
		~c_3\cdot\frac{\log^2\left(8n^2\bar{\alpha}^2\right)}{\bar{\alpha}} ~&~ n\bar{\alpha} ~>~ \frac12~
		\end{cases}~,
		\]
	\end{itemize}
	where $\bar{\alpha}=\frac{1}{n}\sum_{i=1}^{n}\alpha_i$ and $c_1,c_2,c_3>0$ are universal constants. 
\end{proposition}

\begin{proof}
	The first part of the proof focuses on bounding the first term in absolute value for the case $\frac{1}{\bar{\alpha}}\le n^3\bar{\alpha}^2$. It is a minor refinement of Lemma 8 in \citet{safran2020good}.
	Define
	\begin{equation}\label{eq:Y_j2}
		~Y_j ~\coloneqq~ \beta_{\sigma(j)}\prod_{i=j+1}^{n}(1-\alpha_{\sigma(i)})~.
	\end{equation}
	Assuming $n\bar{\alpha}\le\frac12$, we have from \citet[Equations~(31),(32)]{safran2020good} that
	\begin{equation}\label{eq:Y_expectation}
	~\E\pcc{Y_j} ~=~ \sum_{m=1}^{n-j}(-1)^m \E\pcc{\beta_{\sigma(j)}\sum_{j+1\le i_1,\ldots,i_m\le n\text{ distinct}}~\prod_{l=1}^{m}\alpha_{\sigma(i_l)}}~,
	\end{equation}
	and
	\begin{align*}
	&\E\pcc{\beta_{\sigma(j)}\sum_{j+1\le i_1,\ldots,i_m\le n\text{ distinct}}~\prod_{l=1}^{m}\alpha_{\sigma(i_l)}} \\
	&~~~~~~~~~~=~ -\frac{(n-m)!}{n!} \sum_{t_1\in[n]}~ \sum_{t_2\in[n]\setminus\set{t_1}} \ldots \sum_{t_m\in[n]\setminus\set{t_1,\ldots,t_{m-1}}}
	\alpha_{t_1}\alpha_{t_2}\ldots \alpha_{t_m}~\frac{1}{n-m}\sum_{t_{m+1}\in\set{t_1,\ldots,t_m}} \beta_{t_{m+1}}~\\
	&~~~~~~~~~~=~ -\frac{1}{n-m}\binom{n}{m}^{-1}\sum_{1\le t_1 < \ldots < t_m\le n} \alpha_{t_1}\alpha_{t_2}\ldots \alpha_{t_m} \sum_{t_{m+1}\in\set{t_1,\ldots,t_m}} \beta_{t_{m+1}}~.
	\end{align*}
	Using Maclaurin's inequality and the assumption $\beta_j\le1$ for all $j$, the above is upper bounded in absolute value by
	\[
	~\frac{m}{n-m}\bar{\alpha}^m~.
	\]
	Plugging this in \eqref{eq:Y_j2} yields
	\begin{align*}
	~\abs{\E\pcc{Y_j}} ~&\le~ \sum_{m=1}^{n-j}\abs{(-1)^m \frac{m}{n-m}\bar{\alpha}^m} ~\le~ \sum_{m=1}^{n-1} \frac{m}{n-m}\bar{\alpha}^m ~\\
	&\le~ \sum_{m=1}^{n-1} n^{m-1}\bar{\alpha}^m ~=~ \bar{\alpha}\frac{1-(\bar{\alpha}n)^{n-1}}{1-\bar{\alpha}n} ~\le~ 2\bar{\alpha}~,
	\end{align*}
	where the third inequality is due to $n\ge2$ and the last inequality is by the assumption $n\bar{\alpha}\le\frac12$. Lastly, we plug the above in \eqref{eq:Y_expectation} to obtain
	\begin{equation}\label{eq:abs_value_expec_bound}
	~\E\pcc{\abs{\sum_{j=1}^{n}\beta_{\sigma(j)}\prod_{i=j+1}^{n}(1-\alpha_{\sigma(i)})}} ~\le~ 2n\bar{\alpha}~.
	\end{equation}
	Assuming $n\bar{\alpha}>\frac12$, we have from \lemref{lem:keyup} with probability at least $1-\frac{1}{n\sqrt{2\bar{\alpha}}}>0$ that
	\begin{align}
	~\p{\sum_{j=1}^{n}\beta_{\sigma(j)}\prod_{i=j+1}^{n}(1-\alpha_{\sigma(i)})}^2 ~&\le~ c\cdot\log^2\left(\sqrt{2\bar{\alpha}}\cdot8n^{2}\right)\cdot \min\left\{\frac{1}{\bar{\alpha}}~,~n^3\bar{\alpha}^2\right\}\nonumber\\
	&\le~ c\cdot\log^2\left(\sqrt{2\bar{\alpha}}\cdot8n^{2}\right)\cdot \frac{1}{\bar{\alpha}}~.\label{eq:whp_square_bound}
	\end{align}
	Compute using the law of total expectation and the square root of the above equation, using the fact that the square root of the above quantity is deterministically upper bounded by $n$ due to the assumptions $\beta_j\le1$ and $\alpha_j\le1$
	\begin{align*}
	~\E\pcc{\abs{\sum_{j=1}^{n}\beta_{\sigma(j)}\prod_{i=j+1}^{n}(1-\alpha_{\sigma(i)})}} ~&\le~ n\cdot\frac{1}{n\sqrt{2\bar{\alpha}}} + \sqrt{c}\cdot\log\left(\sqrt{2\bar{\alpha}}\cdot8n^{2}\right)\cdot \frac{1}{\sqrt{\bar{\alpha}}} \p{1-\frac{1}{n\sqrt{2\bar{\alpha}}}}~\\
	&\le~ \cdot\frac{1}{\sqrt{2\bar{\alpha}}} + \log\left(\sqrt{2\bar{\alpha}}\cdot8n^{2}\right)\cdot \frac{\sqrt{c}}{\sqrt{\bar{\alpha}}} ~\le~ \frac{c_1\log\left(\sqrt{2\bar{\alpha}}\cdot8n^{2}\right)}{\sqrt{\bar{\alpha}}}~,
	\end{align*}
	for some constant $c_1>0$. Combining the above with \eqref{eq:abs_value_expec_bound} completes the first part of the proposition. Moving to the second assuming $n\bar{\alpha}\le\frac12$, we have again from \lemref{lem:keyup} with probability at least $1-n\bar{\alpha}^2>0$ that
	\begin{equation*}
	~\p{\sum_{j=1}^{n}\beta_{\sigma(j)}\prod_{i=j+1}^{n}(1-\alpha_{\sigma(i)})}^2 ~\le~ c\cdot\log^2\left(\frac{8}{n\bar{\alpha}}\right)\cdot \min\left\{\frac{1}{\bar{\alpha}}~,~n^3\bar{\alpha}^2\right\}\nonumber~.
	\end{equation*}
	From the law of total expectation, the above, and the fact that the quantity above is deterministically upper bounded by $n^2$ due to the assumptions $\beta_j\le1$ and $\alpha_j\le1$, we have
	\begin{align}
	~\E\pcc{\p{\sum_{j=1}^{n}\beta_{\sigma(j)}\prod_{i=j+1}^{n}(1-\alpha_{\sigma(i)})}^2} ~&\le~ n^2\cdot n\bar{\alpha}^2 + c\cdot\log^2\left(\frac{8}{n\bar{\alpha}}\right)\cdot \min\left\{\frac{1}{\bar{\alpha}}~,~n^3\bar{\alpha}^2\right\}\cdot\p{1-n\bar{\alpha}^2}~\nonumber\\
	&\le~ n^3\bar{\alpha}^2 + c\cdot\log^2\left(\frac{8}{n\bar{\alpha}}\right)\cdot n^3\bar{\alpha}^2 = c_2\cdot\log^2\left(\frac{8}{n\bar{\alpha}}\right)\cdot n^3\bar{\alpha}^2~,\label{eq:square_expec_bound}
	\end{align}	
	for some constant $c_2>0$. Likewise, assuming $n\bar{\alpha}>\frac12$, we have from \lemref{lem:keyup} with probability at least $1-\frac{1}{n^2\bar{\alpha}}>0$
	\begin{equation*}
	~\p{\sum_{j=1}^{n}\beta_{\sigma(j)}\prod_{i=j+1}^{n}(1-\alpha_{\sigma(i)})}^2 ~\le~ c\cdot\log^2\left(8n^2\bar{\alpha}^2\right)\cdot \min\left\{\frac{1}{\bar{\alpha}}~,~n^3\bar{\alpha}^2\right\}\nonumber~.
	\end{equation*}
	From the law of total expectation, the above, and the fact that the quantity above is deterministically upper bounded by $n^2$ due to the assumptions $\beta_j\le1$ and $\alpha_j\le1$, we have
	\begin{align*}
	~\E\pcc{\p{\sum_{j=1}^{n}\beta_{\sigma(j)}\prod_{i=j+1}^{n}(1-\alpha_{\sigma(i)})}^2} ~&\le~ n^2\cdot \frac{1}{n^2\bar{\alpha}} + c\cdot\log^2\left(8n^2\bar{\alpha}^2\right)\cdot \min\left\{\frac{1}{\bar{\alpha}}~,~n^3\bar{\alpha}^2\right\}\cdot\p{1-\frac{1}{n^2\bar{\alpha}}}~\\
	&\le~ \frac{1}{\bar{\alpha}} + c\cdot\log^2\left(8n^2\bar{\alpha}^2\right)\cdot \frac{1}{\bar{\alpha}} = c_3\cdot\frac{\log^2\left(8n^2\bar{\alpha}^2\right)}{\bar{\alpha}}~,
	\end{align*}
	for some constant $c_3>0$. Combining the above with \eqref{eq:square_expec_bound} completes the proof of the proposition.
\end{proof}

\subsection{Proof of \lemref{lem:keyup}}\label{subsec:lemkeyupproof}
\begin{proof}	
	We will upper bound the expression $\left(\sum_{j=1}^{n}\beta_{\sigma(j)}\prod_{i=j+1}^{n}(1-\alpha_{\sigma(i)}\right)^2$ in two different manners. Taking the minimum of the two will lead to the desired bound. 
	
	First, using summation by parts and the fact that $\sum_{j=1}^{n}\beta_{\sigma(j)}=0$, we have that
	\begin{align*}
	\left|\sum_{j=1}^{n}\beta_{\sigma(j)}\prod_{i=j+1}^{n}(1-\alpha_{\sigma(i)})\right|
	&=~ \left|\sum_{j=1}^{n}\beta_{\sigma(j)}-\sum_{j=1}^{n-1}\left(\prod_{i=j+2}^{n}(1-\alpha_{\sigma(i)})-\prod_{i=j+1}^{n}(1-\alpha_{\sigma(i)})\right)\sum_{i=1}^{j}\beta_{\sigma(i)}\right|\\
	&=~ \left|\sum_{j=1}^{n-1}\alpha_{\sigma(j+1)}\prod_{i=j+2}^{n}(1-\alpha_{\sigma(i)})\sum_{i=1}^{j}\beta_{\sigma(i)}\right|~\leq~
	\sum_{j=1}^{n-1}\alpha_{\sigma(j+1)}\left|\sum_{i=1}^{j}\beta_{\sigma(i)}\right|~.
	\end{align*}
	By the Hoeffding-Serfling bound and a union bound, we have that with probability at least $1-\delta$, it holds simultaneously for all $j\in \{1,\ldots,n\}$ that $|\sum_{i=1}^{j}\beta_{\sigma(i)}|\leq \sqrt{\log(2n/\delta)j/2}\leq \sqrt{\log(2n/\delta)n/2}$. Plugging into the above, we get that with probability at least $1-\delta$, 
	\begin{align}
	&\left(\sum_{j=1}^{n}\beta_{\sigma(j)}\prod_{i=j+1}^{n}(1-\alpha_{\sigma(i)})\right)^2~\leq~\left(\sum_{j=1}^{n-1}\alpha_{\sigma(j+1)}\sqrt{\frac{\log(2n/\delta)n}{2}}\right)^2\notag\\
	&~~~~~\leq~ \left(n\bar{\alpha}\cdot \sqrt{\frac{\log(2n/\delta)n}{2}}\right)^2
	~=~ \frac{\log(2n/\delta)}{2}\cdot n^3\bar{\alpha}^2~.\label{eq:firstbound} 
	\end{align}
	
	We now turn to upper bound the expression $\left(\sum_{j=1}^{n}\beta_{\sigma(j)}\prod_{i=j+1}^{n}(1-\alpha_{\sigma(i)}\right)^2$ in a different manner. To that end, define the index
	\begin{equation}\label{eq:rdef}
	r~:=~\min\left\{n~,~\left\lceil \frac{6\log(n/\delta)}{\bar{\alpha}}\right\rceil\right\}~\in~ \{1,\ldots,n\}~.
	\end{equation}
	We first show that $\sum_{j=1}^{n}\beta_{\sigma(j)}\prod_{i=j+1}^{n}(1-\alpha_{\sigma(i)})$ is close to $\sum_{j=n-r+1}^{n}\beta_{\sigma(j)}\prod_{i=j+1}^{n}(1-\alpha_{\sigma(i)})$ (namely, where we sum only the last $r$ terms). This is trivially true if $r=n$, so let us focus on the case $r<n$, in which case $r=\left\lceil \frac{6\log(n/\delta)}{\bar{\alpha}}\right\rceil$. We begin by noting that
	\[
	\prod_{i=1}^{r}(1-\alpha_{\sigma(i)})~=~\exp\left(\sum_{i=1}^{r}\log(1-\alpha_{\sigma(i)})\right)~\leq~ \exp\left(-\sum_{i=1}^{r}\alpha_{\sigma(i)}\right)~.
	\]
	Noting that $\frac{1}{n}\sum_{i=1}^{n}\alpha_i^2\leq \frac{1}{n}\sum_{i=1}^{n}\alpha_i=\bar{\alpha}$, and using Bernstein's inequality (applied to sampling without replacement, see for example \citet[Corollary 3.6]{bardenet2015concentration}), we have that for any $r\in \{1,\ldots,n\}$, with probability at least $1-\delta$,  it holds that
	\[
	\frac{1}{r}\sum_{i=1}^{r}\alpha_{\sigma(i)}~\geq~ \bar{\alpha}-\sqrt{\frac{2\bar{\alpha}\log(1/\delta)}{r}}-\frac{\log(n/\delta)}{r}~\geq~ \frac{\bar{\alpha}}{2}-\frac{3\log(1/\delta)}{r}~,
	\]
	where in the last inequality we used the fact $\sqrt{2xy}\leq \frac{x}{2}+2y$ for $x,y\geq 0$. Plugging into the previous displayed equation, we have that with probability at least $1-\delta$,
	\[
	\prod_{i=1}^{r}(1-\alpha_{\sigma(i)})~\leq~\exp\left(-\frac{r\bar{\alpha}}{2}+3\log\left(\frac{1}{\delta}\right)\right)~.
	\]
	Recalling that we assume $r=\left\lceil \frac{6\log(n/\delta)}{\bar{\alpha}}\right\rceil\geq \frac{6\log(n/\delta)}{\bar{\alpha}}$ and plugging into the above, it follows that
	\[
	\prod_{i=1}^{r}(1-\alpha_{\sigma(i)})\leq \exp(-3\log(n))~=~\frac{1}{n^3}~.
	\]
	Since $\sigma$ is a permutation, the same upper bound holds with the same probability for $\prod_{i=n-r+1}^{n}(1-\alpha_{\sigma(i)})$. Thus, we have that with probability at least $1-\delta$,
	\begin{align}
	\left|\sum_{j=1}^{n}\beta_{\sigma(j)}\prod_{i=j+1}^{n}(1-\alpha_{\sigma(i)})\right|&~\leq~\left|\sum_{j=n-r+1}^{n}\beta_{\sigma(j)}\prod_{i=j+1}^{n}(1-\alpha_{\sigma(i)})\right|+\sum_{j=1}^{n-r}|\beta_{\sigma(j)}|\prod_{i=j+1}^{n}(1-\alpha_{\sigma(i)})\notag\\
	&\leq~ \left|\sum_{j=n-r+1}^{n}\beta_{\sigma(j)}\prod_{i=j+1}^{n}(1-\alpha_{\sigma(i)})\right|+\sum_{j=1}^{n-r} 1\cdot \prod_{i=n-r+1}^{n}(1-\alpha_{\sigma(i)})\notag\\
	&\leq~
	\left|\sum_{j=n-r+1}^{n}\beta_{\sigma(j)}\prod_{i=j+1}^{n}(1-\alpha_{\sigma(i)})\right|+\sum_{j=1}^{n-r} 1\cdot \frac{1}{n^3}\notag\\
	&~\leq~\left|\sum_{j=n-r+1}^{n}\beta_{\sigma(j)}\prod_{i=j+1}^{n}(1-\alpha_{\sigma(i)})\right|+\frac{1}{n^2}~.\label{eq:bound1}
	\end{align}
	We showed this assuming $r<n$, but the same overall inequality trivially also holds for $r=n$ (with probability $1$). Therefore, the inequality holds regardless of the value of $r$ (as defined in \eqref{eq:rdef}).
	
	To further upper bound this, we note that every term $\beta_{\sigma(j)}\prod_{i=j+1}^{n}(1-\alpha_{\sigma(i)})$ in the sum above has magnitude at most $1$. Applying Azuma's inequality on the martingale difference sequence $\beta_{\sigma(j)}\prod_{i=j+1}^{n}(1-\alpha_{\sigma(i)})-\E\left[\beta_{\sigma(j)}\prod_{i=j+1}^{n}(1-\alpha_{\sigma(i)})\middle|\sigma(j+1),\ldots,\sigma(n)\right]$ (indexed by $j$ going down from $n$ to $n-r+1$), we have that with probability at least $1-\delta$,
	\begin{align}
	&\left|\sum_{j=n-r+1}^{n}\left(\beta_{\sigma(j)}\prod_{i=j+1}^{n}(1-\alpha_{\sigma(i)})-\E\left[\beta_{\sigma(j)}\prod_{i=j+1}^{n}(1-\alpha_{\sigma(i)})~\Big|~\sigma(j+1),\ldots,\sigma(n)\right]\right)\right|\notag\\
	&\;\;\;\;\;\;\;~\leq~ \sqrt{2r\log\left(\frac{2}{\delta}\right)}~.\label{eq:bound2}
	\end{align}
	Furthermore, since $\sigma$ is a random permutation and $\frac{1}{n}\sum_{i=1}^{n}\beta_i=0$, the following holds with probability at least $1-\delta$ simultaneously for all $j\in \{1,\ldots,n\}$, by the Hoeffding-Serfling bound and a union bound:
	\begin{align*}
	\left|\E\left[\beta_{\sigma(j)}\prod_{i=j+1}^{n}(1-\alpha_{\sigma(i)})~\Big|~\sigma(j+1),\ldots,\sigma(n)\right]\right|~&=~\left|\prod_{i=j+1}^{n}(1-\alpha_{\sigma(i)})\cdot \frac{1}{j}\sum_{i\in \{1,\ldots,n\}\setminus\{\sigma(j+1),\ldots,\sigma(n)\}}\beta_{i}\right|\\
	&\leq~ \left|\frac{1}{j}\sum_{i\in \{1,\ldots,n\}\setminus\{\sigma(j+1),\ldots,\sigma(n)\}}\beta_{i}\right|\\
	&\leq~ \sqrt{\frac{\log(2n/\delta)}{2j}}~.
	\end{align*}
	Combining the above together with \eqref{eq:bound1} and \eqref{eq:bound2} (using a union bound), we get overall that with probability at least $1-3\delta$,
	\[
	\left|\sum_{j=1}^{n}\beta_{\sigma(j)}\prod_{i=j+1}^{n}(1-\alpha_{\sigma(i)})\right|~\leq~ \frac{1}{n^2}+\sqrt{2r\log\left(\frac{2}{\delta}\right)}+\sum_{j=n-r+1}^{n}\sqrt{\frac{\log(2n/\delta)}{2j}}~.
	\]
	Noting\footnote{To see this, note that if $r\geq n/2$, then 
		$\sum_{j=n-r+1}^{n}\sqrt{\frac{1}{j}}\leq \sum_{j=1}^{n}\sqrt{\frac{1}{j}}\leq 2\sqrt{n}\leq 2\sqrt{2r}$, and if $r<n/2$, then $\sum_{j=n-r+1}^{n}\sqrt{\frac{1}{j}}\leq \frac{r}{\sqrt{n-r+1}}\leq \frac{r}{\sqrt{n-n/2+1}}\leq \frac{r}{\sqrt{n/2}}\leq \frac{r}{\sqrt{r}}=\sqrt{r}$.} that $\sum_{j=n-r+1}^{n}\sqrt{\frac{1}{j}}\leq 2\sqrt{2r}$, plugging into the above and simplifying a bit, we get that with probability at least $1-3\delta$,
	\[
	\left|\sum_{j=1}^{n}\beta_{\sigma(j)}\prod_{i=j+1}^{n}(1-\alpha_{\sigma(i)})\right|~\leq~5\sqrt{r\cdot\log(2n/\delta)}~.
	\]
	Squaring both sides, plugging in the definition of $r$ and further simplifying a bit, we get that with probability at least $1-3\delta$,
	\[
	\left(\sum_{j=1}^{n}\beta_{\sigma(j)}\prod_{i=j+1}^{n}(1-\alpha_{\sigma(i)})\right)^2~\leq~ c'\log^2\left(\frac{2n}{\delta}\right)\cdot\min\left\{n,\frac{1}{\bar{\alpha}}\right\}
	\]
	for some universal constant $c'>1$. Combining this with \eqref{eq:firstbound} using a union bound, we have that with probability at least $1-4\delta$, 
	\[
	\left(\sum_{j=1}^{n}\beta_{\sigma(j)}\prod_{i=j+1}^{n}(1-\alpha_{\sigma(i)})\right)^2~\leq~
	c'\cdot\log^2\left(\frac{2n}{\delta}\right)\cdot\min\left\{n,\frac{1}{\bar{\alpha}},n^3\bar{\alpha}^2\right\}.
	\]
	Finally, noting that $\min\left\{n,\frac{1}{\bar{\alpha}},n^3\bar{\alpha}^2\right\}=\min\left\{\frac{1}{\bar{\alpha}},n^3\bar{\alpha}^2\right\}$, and letting $\delta':=4\delta$, we get that with probability at least $1-\delta'$, the expression above is at most $c_3\log^2(8n/\delta')\min\left\{\frac{1}{\bar{\alpha}},n^3\bar{\alpha}^2\right\}$ as required.
\end{proof}

\subsection{Remaining Technical Proofs}

\begin{lemma}\label{lem:deterministic_prod}
	Suppose $\sigma_0,\ldots,\sigma_{n-1}$ is a permutation of $\frac n2$ $0$'s and $\frac n2$ $1$'s, and that $\eta\le\frac{1}{\lambdamax n}$. Then
	\[
	~\prod_{i=0}^{n-1}(1-\eta\lambdamax\sigma_i) ~\ge~ 1-\frac{\eta\lambdamax n}{2} ~\ge~ \frac12~.
	\]
\end{lemma}

\begin{proof}
	The proof follows immediately from Bernoulli's inequality and the assumption $\eta\le\frac{1}{\lambdamax n}$.
\end{proof}

\begin{lemma}\label{lem:prod_expect}
	Suppose $\sigma_0,\ldots,\sigma_{n-1}$ is a random permutation of $\frac n2$ $0$'s and $\frac n2$ $1$'s. 
	Then for all $m\in\{1,\ldots,n-1\}$
	\[
	~\E_{\sigma}\pcc{(1-2\sigma_0)\prod_{i=1}^{m}\sigma_i} ~=~ \frac12\binom{n/2-1}{m-1}\binom{n-1}{m}^{-1}~.
	\]
\end{lemma}

\begin{proof}
	Compute
	\begin{align*}
	~\E\pcc{(1-2\sigma_0)\prod_{i=1}^{m}\sigma_i} ~&=~ \frac12\E\pcc{\prod_{i=1}^{m}\sigma_i\Big|\sigma_0=0} - \frac12\E\pcc{\prod_{i=1}^{m}\sigma_i\Big|\sigma_0=1}\\
	&=~ \frac12\Pr\pcc{\sigma_1=\ldots=\sigma_m=1\Big|\sigma_0=0} - \frac12\Pr\pcc{\sigma_1=\ldots=\sigma_m=1\Big|\sigma_0=1}\\
	&=~ \frac12\p{\frac{n/2\cdot(n/2-1)\cdot\ldots\cdot(n/2-m+1)}{(n-1)\cdot(n-2)\cdot\ldots\cdots(n-m)}}\\
	&~~~~~~~-\frac12\p{\frac{(n/2-1)\cdot(n/2-2)\cdot\ldots\cdot(n/2-m)}{(n-1)\cdot(n-2)\cdot\ldots\cdot(n-m)}}\\
	&=~ \frac12\p{\frac{(n/2-1)\cdot(n/2-2)\cdot\ldots\cdot(n/2-m+1)}{(n-1)\cdot(n-2)\cdot\ldots\cdot(n-m)}}m\\
	&=~ \frac12\cdot\frac{(n/2-1)!}{(n/2-m)!}\cdot\frac{(n-m-1)!}{(n-1)!}m\\
	&=~\frac12\binom{n/2-1}{m-1}\binom{n-1}{m}^{-1}~.
	\end{align*}
\end{proof}

\begin{lemma}\label{lem:sum_prod_expect}
	Suppose $\sigma_0,\ldots,\sigma_{n-1}$ is a random permutation of $\frac n2$ $0$'s and $\frac n2$ $1$'s and $\eta\le\frac{1}{\lambdamax n}$.
	Then
	\[
	~\E\pcc{\sum_{i=0}^{n-1}(1-2\sigma_i) \prod_{j=i+1}^{n-1} (1-\eta\lambdamax\sigma_j)} ~\le~ -\frac18\eta\lambdamax n~. 
	\]
\end{lemma}

\begin{proof}
	Denote $ Y_i \coloneqq (1-2\sigma_i)\p{\prod_{j=i+1}^{n-1}\p{1-\eta\lambdamax\sigma_j}} $, we expand $ Y_i $ to obtain
	\begin{align*}
	\E\pcc{Y_i} ~&=~ 
	\E\pcc{1-2\sigma_i} + \sum_{m=1}^{n-i}(-\eta\lambdamax)^m \E\pcc{\sum_{i+1\le i_1,\ldots,i_m\le n-1\text{ distinct}}~(1-2\sigma_i)\p{\prod_{l=1}^{m}\sigma_{i_l}}} \nonumber\\
	~&=~\sum_{m=1}^{n-i-1}(-\eta\lambdamax)^m \binom{n-i-1}{m} \E\pcc{ (1-2\sigma_0)\prod_{l=1}^{m}\sigma_l}\\
	~&=~ \frac12 \sum_{m=1}^{n-i-1}(-\eta\lambdamax)^m \binom{n-i-1}{m} \binom{n/2-1}{m-1}\binom{n-1}{m}^{-1},
	\end{align*}
	where the last equality is by \lemref{lem:prod_expect}. Denote the $m$-th summand by $a_m$, we bound the quotient of two subsequent terms in the above sum by using the assumption $\eta\le\frac{1}{\lambdamax n}$, so we get for any $m\ge1$
	\[
	~\abs{\frac{a_{m+1}}{a_m}} ~\le~ \eta\lambdamax\frac{(n-2m)(n-m-i-1)}{2m(n-m-1)} ~\le~ \frac12\eta\lambdamax n ~\le~ \frac12~,
	\]
	thus the above sum which alternates signs and begins with a negative term is upper bounded by
	\[
	\frac12(a_1+a_2) ~\le~ \frac14a_1 ~=~ -\frac14\eta\lambdamax\frac{n-i-1}{n-1}~, 
	\]
	where we conclude with
	\begin{align*}
	~\E\pcc{\sum_{i=0}^{n-1}(1-2\sigma_i) \prod_{j=i+1}^{n-1} (1-\eta\lambdamax\sigma_j)} ~&=~ \sum_{i=0}^{n-1}\E\pcc{Y_i} ~\le~ -\frac14\eta\lambdamax\sum_{i=0}^{n-1}\frac{n-i-1}{n-1}\\
	&=~ -\frac18\eta\lambdamax\frac{n(n-1)}{n-1} ~=~ -\frac18\eta\lambdamax n~.
	\end{align*}
\end{proof}

\section{Equivalence of Optimization Under Conjugate Transformations}\label{app:conj}

Suppose we are given an orthogonal matrix $O$, an initialization point $\bx_0\in\reals^d$ and an optimization problem $F(\bx)\coloneqq\frac12\bx^{\top}A\bx-\bb^{\top}\bx=\frac1n\sum_{i=1}^{n}f_i(\bx)$ where $f_i(\bx)=\frac12\bx^{\top}A_i\bx-\bb_i^{\top}\bx$. Define the $O$-conjugate optimization problem as $\tilde{F}(\bx)\coloneqq\frac12\bx^{\top}\tilde{A}\bx - \tilde{\bb}^{\top}\bx=\frac1n\sum_{i=1}^{n}\tilde{f}_i(\bx)$ where $\tilde{f}_i(\bx)=\frac12\bx^{\top}\tilde{A}_i\bx-\tilde{\bb}_i^{\top}\bx$, initialized from $\tilde{\bx}_0$, whereby $\tilde{A}_i,\tilde{\bb}_i,\tilde{\bx}_0$ are defined using the following transformations:
\[
~\tilde{A}_i ~\coloneqq~ OAO^{\top}~, ~~~~~\tilde{\bb}_i ~\coloneqq~ O\bb_i~, ~~~~~\tilde{\bx}_0 ~\coloneqq~ O\bx_0~.
\]
In this appendix, we show that the $O$-conjugate optimization problem is equivalent in terms of the sub-optimality rate of without-replacement SGD. More formally, we have the following theorem:
\begin{theorem}
	Suppose we have $F$, $\tilde{F}$ and $O$ as above. Let $\bx_t$ and $\tilde{\bx}_t$ denote the iterate after performing $t$ steps of without-replacement SGD. Then
	\[
	~\tilde{\bx}_t ~=~ O\bx_t~.
	\]
	And in particular, we have that
	\[
	~\tilde{F}(\tilde{\bx}_t) ~=~ F(\bx_t)~.
	\]
	Moreover, if $F$ satisfies Assumption \ref{as:upper_bounds} then so does $\tilde{F}$.
\end{theorem}

\begin{proof}
	Using induction, the base case is immediate from the definition of $\tilde{\bx}_0$, and we have
	\[
	~\tilde{F}(\tilde{\bx}_0) ~=~ \frac12\bx_0^{\top}O^{\top}OAO^{\top}O\bx_0 - \bb^{\top}O^{\top}O\bx_0 ~=~ F(\bx_0)~.
	\]
	For the induction step, assume the theorem is true for $t$. We will show it also holds for $t+1$. Compute for all $i\in[n]$
	\[
	~\nabla_{\bx}\tilde{f}_i(\bx) ~=~ \nabla_{\bx}\p{\frac12\bx^{\top}OA_iO^{\top}\bx - (O\bb_i)^{\top}\bx} ~=~ OA_iO^{\top}\bx - O\bb_i~.
	\]
	Suppose the next function to be processed in iteration $t+1$ is $f_i$, the update rule of without-replacement SGD therefore satisfies
	\[
	~\tilde{\bx}_{t+1} ~=~ \tilde{\bx}_{t} - \eta\nabla_{\bx}f_i(\tilde{\bx}_{t}) ~=~ O\bx_t - \eta OA_iO^{\top}O\bx_t + \eta O\bb_i ~=~ O\p{\bx_t-\eta (A_i\bx_t - \bb_i)} ~=~ O\bx_{t+1} ~.
	\]
	Plugging the above in $\tilde{F}$ we obtain
	\[
	~\tilde{F}(\tilde{\bx}_{t+1}) ~=~ \frac12\bx_{t+1}^{\top}O^{\top}OAO^{\top}O\bx_{t+1} - \bb_i^{\top}O^{\top}O\bx_{t+1} ~=~ F(\bx_{t+1})~.
	\]
	Lastly, it is readily seen that if $F$ satisfies Assumption \ref{as:upper_bounds} then so does $\tilde{F}$ since orthogonal matrices are isometries.
\end{proof}

\end{document}